\documentclass[10pt]{article}
\usepackage[accepted]{rlj}
\usepackage{amssymb}
\usepackage{mathtools}
\usepackage{graphicx}
\usepackage{subcaption}
\usepackage[space]{grffile}
\usepackage{url}
\usepackage{algorithm}
\usepackage{algorithmic}
\usepackage{booktabs}
\usepackage{multirow}
\usepackage{pgfplots}
\usepackage{tikz}
\usepackage{pifont}
\usepackage{colortbl}
\usepackage{amsthm}
\usepackage{enumitem}
\usepackage{pgfplots}
\usepackage{subcaption} 
\usetikzlibrary{pgfplots.groupplots}
\pgfplotsset{compat=1.18}

\definecolor{pogp}{RGB}{55,126,184}    
\definecolor{dql}{RGB}{228,26,28}      
\definecolor{td3}{RGB}{77,175,74}      
\usepackage{amsmath,amssymb}
\usepackage{xcolor}

\usetikzlibrary{
    positioning,
    arrows.meta,
    shapes.geometric,
    shapes.misc,
    fit,
    backgrounds,
    calc,
    decorations.pathreplacing,
    patterns
}

\newtheorem{definition}{Definition}
\newtheorem{theorem}{Theorem}
\newtheorem{proposition}{Proposition}
\newtheorem{corollary}{Corollary}
\newtheorem{remark}{Remark}

\definecolor{pogpc}{HTML}{D62728}
\definecolor{dqlc}{HTML}{1F77B4}
\definecolor{sacc}{HTML}{FF7F0E}
\definecolor{td3c}{HTML}{2CA02C}
\definecolor{ppoc}{HTML}{9467BD}
\definecolor{dppoc}{HTML}{8C564B}
\definecolor{sdacc}{HTML}{E377C2}
\definecolor{fqlc}{HTML}{BCBD22}
\definecolor{chaincolor}{HTML}{4E79A7}
\definecolor{stopcolor}{HTML}{E15759}
\definecolor{contcolor}{HTML}{76B7B2}

\usepackage{tikz}
\usetikzlibrary{arrows.meta, positioning, fit, backgrounds}
\usepackage{xcolor}

\definecolor{brandEmerald}{HTML}{10B981}
\definecolor{brandNavy}{HTML}{0F172A}
\definecolor{brandSlate}{HTML}{64748B}
\definecolor{bgGray}{HTML}{F1F5F9}

\tikzset{
  arrow/.style={-Latex, thick},
  data/.style={draw, rounded corners, align=center, minimum width=28mm, minimum height=9mm, fill=white},
  block/.style={draw, rounded corners, align=center, minimum width=40mm, minimum height=10mm, fill=white},
  smallblock/.style={draw, rounded corners, align=center, minimum width=30mm, minimum height=8mm, fill=white},
  op/.style={draw, diamond, aspect=2.0, inner sep=1pt, align=center, fill=white}
}

\usepackage{tikz}
\usetikzlibrary{arrows.meta, positioning, fit, backgrounds}

\usepackage{xcolor}
\definecolor{brandEmerald}{HTML}{10B981}
\definecolor{brandNavy}{HTML}{0F172A}
\definecolor{brandSlate}{HTML}{64748B}
\definecolor{bgGray}{HTML}{F1F5F9}

\tikzset{
  arrow/.style={-Latex, thick},
  data/.style={draw, rounded corners, align=center, minimum width=26mm, minimum height=8mm, fill=white},
  block/.style={draw, rounded corners, align=center, minimum width=34mm, minimum height=10mm, fill=white},
  op/.style={draw, diamond, aspect=2, inner sep=1.2pt, align=center, fill=white}
}

\title{Learning When to Stop: Prefix-Optimal Dynamic Diffusion Policies for Continuous Control}

\setrunningtitle{Prefix-Optimal Generative Policies}

\author{Rohit Kumar Salla, Manoj Saravanan, Simon Stepputtis}
\emails{rohits25@vt.edu, manoj663@vt.edu, stepputtis@vt.edu}
\affiliations{Virginia Polytechnic Institute and State University (Virginia Tech),
Blacksburg, Virginia, USA}

\contribution{%
  \textbf{Prefix Value Function:} A value function trained via TD bootstrapping over the denoising chain that serves as an auxiliary objective improving intermediate action quality.%
}{%
  Diffusion policies optimize only the terminal action, leaving intermediate iterates unsupervised and non-executable.%
}

\contribution{%
  \textbf{Learned Early Stopping:} A stopping rule requiring no separately learned halting network derived from consecutive prefix value differences that halts denoising when further refinement is predicted to yield negligible improvement.%
}{%
  Prior acceleration methods either apply fixed step reductions or train separate adapter networks decoupled from intermediate action quality.%
}

\contribution{%
  \textbf{Empirical Validation:} Prefix training improves full-chain performance by approximately 4.6\% over state-of-the-art baselines while enabling $\approx 2.7\times$ compute savings across four MuJoCo environments and 12 baselines.%
}{%
  The dual benefit suggests intermediate supervision acts as a beneficial auxiliary objective beyond enabling adaptive compute allocation.%
}


\summary{
Diffusion policies generate actions through a $K$-step denoising
chain, but standard training optimizes only the terminal output.
We introduce POGP, which learns a prefix value function at every
denoising step through a Bellman-style recursion. The prefix value
function improves intermediate action quality and provides a
test-time signal for adaptive stopping. Across four MuJoCo
environments and 12 baselines, POGP outperforms D3P by
approximately 3.5\% while using 18.2\% fewer iterations.
Relative to a fixed 20-step chain, POGP reduces denoising
computation by approximately $2.7\times$ while retaining
near-full-chain performance.
}

\begin{document}
\makeCover
\maketitle

\begin{abstract}

Diffusion policies have emerged as a powerful policy class for continuous control, but their iterative denoising process remains the primary computational bottleneck. Optimizing the computational effort spent per action is critical for improving the practical adoption of diffusion policies, allowing them to allocate more iterations only when the task demands it. To this end, we propose POGP (Prefix-Optimal Generative Policies), which learns a prefix value function (PVF) at every intermediate denoising step via a Bellman-style recursion over the chain. 
The PVF provides both an auxiliary optimization target that shapes intermediate outputs into high-quality actions and a test-time stopping criterion that halts denoising when further iterations are not expected to yield meaningful improvement. 
Through extensive experiments on four MuJoCo environments against 12 baselines, we demonstrate that POGP reduces the required denoising iterations by approximately $2.7\times$ while retaining near-full performance. Furthermore, compared to current state-of-the-art dynamic diffusion baselines, we show that prefix training improves final task performance by approximately 3.5\%, suggesting that supervising intermediate denoising steps acts as a beneficial auxiliary objective beyond its role in enabling early stopping. \textbf{Project Page}: https://rohitsalla.github.io/POGP

\end{abstract}

\section{Introduction}
\label{sec:introduction}

Diffusion policies for continuous control generate actions through an iterative denoising chain, typically requiring $K$ sequential neural network evaluations per action~\citep{ho2020denoising, diffusionpolicy, wang2022diffusion}. 
Running all $K$ iterations of the diffusion process is the primary computational bottleneck of diffusion policies as every action incurs the same computational cost regardless of the task's difficulty.
In practice, however, many actions taken by an embodied system along some trajectory are routine and well-determined after only a few denoising steps, while others, such as contact-rich transitions, disturbance recovery, or mode switches strongly benefit from extended refinement steps. 
For example, consider an embodied agent perturbed by some external force.
When the disturbance is applied to the system, the policy benefits from additional refinement steps to improve the system's overall behavior while saving compute at other times, which is especially important for local and edge deployments on real robot hardware. 
Current diffusion policy frameworks~\citep{dppo, sdac, wang2024dacer, idql} do not distinguish between these cases, spending compute uniformly across the trajectory.

Recent approaches leveraging diffusion policies~\citep{yu2025d3p} address this issue by estimating the number of iterations needed in order to reach a sufficiently performant action. 
However, because this adapter is independent of the policy, the intermediate denoising steps are never optimized to be executable as the model is only trained on the full diffusion process. 
We take a different approach by jointly training a prefix value function (PVF) alongside the diffusion policy itself, which both improves intermediate action quality as an auxiliary objective and provides a natural per-step stopping signal at test time.

In this work, we propose a novel approach for dynamic diffusion -- Prefix-Optimal Generative Policies (POGP) -- that \textbf{leverages a learned prefix value function (PVF) to assess when diffusion iterations can be \textit{stopped early} while improving the quality of internal diffusion iterations}. 
POGP trains a PVF $V_t(s, a_t)$ for each intermediate action $a_t$ in the $K$-step diffusion chain $a_K \to a_0$ that intuitively measures how valuable it would be to execute action $a_t$ at step $t$.
During training, the PVF is trained to estimate an action's value by completing all $K$ diffusion steps, co-evolving as an optimization target for the diffusion step itself. 
Crucially, rather than regressing each $V_t$ independently against the final action's Q-value, which would require expensive full-chain rollouts per step, we derive a Bellman recursion over the denoising chain that propagates value information backwards through refinement steps, enabling efficient single-chain training with stable TD targets.
Intuitively, the PVF serves a dual purpose: 1) it improves the quality of the intermediate policy actions $a_t$; and 2) provides a metric for stopping the diffusion process early in order to save time and compute budget. 

We demonstrate that, across four MuJoCo tasks and 12 baselines, POGP achieves an IQM of $0.942$ with a $95\%$ confidence interval~\citep{agarwal2021deep}, retains $\geq 97.8\%$ of its full-diffusion process task performance at less than half the computation cost and outperforms current state-of-the-art baselines by $\approx4.6\%$.
In summary, POGP proposes the following contributions:
\begin{itemize}
    \item a prefix value function trained via TD bootstrapping alongside the diffusion model over the denoising chain, ultimately improving the quality of the intermediate steps; 
    \item a principled early stopping rule derived from prefix value differences across consecutive diffusion steps, leveraged to stop the diffusion policy early if the value of future steps is predicted to be minimal; and
    \item empirical evidence that prefix training improves full-chain performance by up to $4.6\%$ while enabling $\approx2.7\times$ compute savings, on average. 
\end{itemize}

\section{Related Work}

\paragraph{Diffusion Policies for Control.}
Diffusion-based policies have emerged as a powerful class of decision models for continuous control and robot learning because they can represent complex, multimodal action distributions that are difficult for unimodal actors to capture. Prior work includes action-diffusion policies for manipulation \citep{diffusionpolicy}, diffusion-based trajectory planning \citep{janner2022planning}, and diffusion-parameterized reinforcement learning methods such as Diffusion-QL~\citep{wang2022diffusion}, Efficient Diffusion Policies~\citep{sdac}, and IDQL~\citep{idql}. A common feature of these methods is that action generation proceeds through a multi-step denoising chain, so inference cost scales with the number of refinement steps. Moreover, existing objectives primarily supervise the terminal denoised action rather than the intermediate prefixes of the chain. As a result, intermediate iterates are generally not trained to be high-value executable decisions. This creates a train--test mismatch whenever inference must be interrupted or truncated under real-time compute constraints. POGP addresses this gap by explicitly assigning value to every prefix of the refinement process through a learned prefix value function thereby turning variable diffusion depth into a first-class object of policy optimization rather than a post hoc deployment heuristic.

\paragraph{Anytime and Adaptive Computation.}
Our work is also related to the broad area on anytime prediction and adaptive computation. It studies systems whose outputs remain useful when computation is stopped early and improve as additional computation is allocated \cite{zilberstein1996anytime,graves2017act,huang2016stochasticdepth,banino2021pondernet}. A central challenge in this setting is that models are often optimized only for full-depth performance, which often yields poor intermediate outputs. While exposing the model to variable computation during training, we can improve robustness across budgets and sometimes regularize full-budget performance. POGP brings this perspective into continuous control, where early outputs are not merely intermediate to representations but executable actions whose quality must be judged by downstream return. This distinction is important because unlike adaptive-computation settings in supervised learning, where intermediate predictions are evaluated by task loss. Here, each refinement prefix induces a policy with its own control performance. Our formulation therefore connects adaptive computation to reinforcement learning through value-based supervision over diffusion prefixes.

\paragraph{Efficient and Dynamic Diffusion.}
Recent efforts to reduce the cost of diffusion inference broadly fall into two categories. The first seeks to improve a predetermined operating point by reducing the number of denoising steps required for good performance. For example through distillation, consistency models or consistency-style policy classes \cite{progressive_distillation,consistency_models,ding2024consistency_policy}. These approaches substantially reduce latency, but they primarily optimize a chosen small-step generator rather than ensuring that arbitrary prefixes of the original diffusion chain are themselves high-value actions. The second line considers adaptive-budget allocation, where the amount of denoising performed is adjusted online. Most closely related is D3P, which introduces a lightweight adaptor to allocate denoising depth at test time \cite{yu2025d3p}. POGP differs in both objective and mechanism: instead of learning a separate stopping module, it trains a prefix value function that directly evaluates the marginal utility of further refinement. The same value-grounded object is used both to improve prefix quality during learning and to determine when additional denoising is no longer worthwhile at inference time.

\section{Background and Notation}
\label{sec:background}

\paragraph{Notation}
We summarize key notation as follows: $s \in \mathcal{S}$ and $a_t \in \mathcal{A}$ represent the environment state and action at refinement level $t$, respectively, where $t \in \{0,\ldots,K\}$ runs backwards ($t{=}K$ is noise, $t{=}0$ is the terminal action). 
The symbol $a_0$ denotes the fully-denoised action.

\paragraph{Markov Decision Process}
We consider discounted continuous-control MDPs $\mathcal{M}=\langle \mathcal{S},\mathcal{A},P,r,\gamma\rangle$ with $\mathcal{S}\subseteq\mathbb{R}^{d_s}$, $\mathcal{A}\subseteq\mathbb{R}^{d_a}$, transition kernel $P(\cdot\mid s,a)$, bounded reward $r:\mathcal{S}\times\mathcal{A}\to[R_{\min},R_{\max}]  $ and discount $\gamma\in(0,1)$. For a (stationary) policy $\pi$, the state-action value is
\begin{equation}
    Q^\pi(s,a)\;:=\;\mathbb{E}_\pi\!\left[\sum_{k=0}^\infty \gamma^k\, r(s_k,a_k)\,\Big|\, s_0=s,\; a_0=a\right],
\end{equation}
and $V^\pi(s):=\mathbb{E}_{a\sim \pi(\cdot\mid s)}[Q^\pi(s,a)]$.

\paragraph{Diffusion Policies}
A diffusion policy represents $\pi(a\mid s)$ implicitly via a $K$-step reverse denoising chain \citep{ho2020denoising,wang2022diffusion}: sample $a_K\sim\mathcal{N}(0,I)$ and apply the denoiser
\begin{equation}
    a_{t-1} \;=\; \pi_\theta(s,a_t,t), \qquad t=K,K-1,\ldots,1,
\end{equation}
to obtain $a_0$. 
In this work, we use a deterministic denoiser (DDIM-style) throughout all experiments. The framework extends to stochastic denoisers by replacing the deterministic refinement transitions with conditional distributions, which we leave as future work.

\paragraph{Prior diffusion RL.}
Diffusion-QL~\citep{wang2022diffusion} and IDQL~\citep{idql} establish diffusion actors as an expressive policy class in offline (and offline-to-online) RL, while recent on-policy and actor-critic variants scale diffusion policies to fully online training~\citep{dppo,wang2024dacer,qvpo}. A shared limitation across these objectives is their \emph{terminal-only} optimization in which they evaluate the environment through the terminal action $a_0$.
This causes intermediate states to transition in an unsupervised way from noise to the terminal action, but introduce the drawback that internal states are not suitable for direct execution. 
However, being able to utilize non-complete diffusion chains is critical for dynamic-budget diffusion policies to allow for early stopping of the iterative generation process. 

\section{Prefix-Optimal Generative Policies}
\label{sec:Method}

Diffusion policies generate actions through a $K$-step denoising chain that transforms random noise into the terminal output action ($a_K \to \cdots \to a_0$). 
Standard training optimizes only the final action $a_0$, leaving intermediate actions $\{a_t\}_{t=1}^{K-1}$ without supervision, causing them to be non-executable. a necessary requirement for dynamically choosing the number of refinement steps.  
POGP addresses this by learning a prefix value function $V_t(s,a_t)$ at every denoising step, providing both a training signal that co-trains with the diffusion policy to shape intermediate outputs to be of high quality, as well as a runtime metric that indicates when further denoising is no longer expected to yield significant improvement.
Learning the prefix value function requires two core components:
1) a formal definition of prefix values grounded in the environment's Q-function; 
and 2) a Bellman-style recursion over the denoising chain that enables efficient TD-based training of $V_t$ from a single shared chain per gradient step. 
An overview of our method is shown in Algorithm~\ref{alg:pogp_main}. 

\begin{algorithm}[t]
\caption{POGP (Prefix-Optimal Generative Policies)}
\label{alg:pogp_main}
\begin{algorithmic}[1]
\STATE \textbf{Input:} Replay buffer $\mathcal{D}$, chain length
$K=20$, uniform hazard rate $h=1/K=0.05$, networks
$\pi_\theta,Q_{\phi_{1,2}},V_\psi,g_\omega$.
\STATE \textbf{Initialize:} Target networks $\phi'_{1,2}, \psi'$.
\FOR{each gradient step}
  \STATE Collect transition $(s,a,r,s')$; add to $\mathcal{D}$.
  \STATE Sample minibatch from $\mathcal{D}$.
  \STATE \textbf{Update} $Q$: $y_Q \leftarrow r + \gamma \min_i Q_{\phi'_i}(s', a'_0)$; minimize $\mathcal{L}_Q$. \hfill $\triangleright$ Critic update
  \STATE Sample $t \sim \mathrm{Unif}\{1,\ldots,K\}$, $a_K \sim \mathcal{N}(0,I)$.
  \STATE Generate chain: $a_K \rightarrow \cdots \rightarrow a_t \rightarrow \cdots \rightarrow a_0$. \hfill $\triangleright$ Shared chain
  \STATE \textbf{Update} $V$: $y_V \leftarrow h\,\min_i Q_{\phi'_i}(s,a_0) + (1-h)\, V_{\psi'}(s,a_{t-1},t{-}1)$; minimize $\mathcal{L}_V$. \hfill $\triangleright$ Prefix Bellman target
  \STATE \textbf{Update} $\pi$: Minimize $\mathcal{L}_\pi$ (Eq.~\ref{eq:actor_loss_final}) with gate $w(n)$. \hfill $\triangleright$ Policy update
  \STATE Sample $T \sim \mathrm{Unif}\{1,\ldots,K\}$; \textbf{update} $g$: minimize $\mathcal{L}_g$ (Eq.~\ref{eq:lg}) using $a_T, a_0$ from the shared chain. \hfill $\triangleright$ Projection update
  \STATE Polyak-update target networks $\phi'_{1,2}, \psi'$.
\ENDFOR
\STATE \textbf{Return:} Trained networks $\pi_\theta, Q_{\phi_{1,2}}, V_\psi, g_\omega$.
\end{algorithmic}
\end{algorithm}

\subsection{The Prefix Value Function}
\label{sec:prefix_value}

We consider iterative refinement policies defined by a deterministic denoiser $\pi_\theta$:
\begin{equation}
a_{t-1} = \pi_\theta(s, a_t, t), \quad a_K \sim \mathcal{N}(0,I), \quad t = K, \ldots, 1,
\label{eq:denoise_chain}
\end{equation}
conditioned on environment state $s$.

To optimize intermediate steps, we require a scalar measure of how ``promising'' a partial result $a_t$ is. The natural criterion is the expected return of completing the chain from $a_t$:

\begin{definition}[Prefix Value]
\label{def:prefix_value}
For environment state $s$ and intermediate action $a_t$ at refinement step $t$:
\begin{equation}
V^\pi_t(s, a_t) \;:=\; Q^\pi\!\bigl(s,\; \pi_{\mathrm{full}}(s, a_t)\bigr),
\label{eq:prefix_value_def}
\end{equation}
where $\pi_{\mathrm{full}}(s, a_t)$ denotes the result of running the denoiser from step $t$ to $0$.
\end{definition}

The prefix value $V_t$ answers: ``If we complete the denoising chain starting from $a_t$, what return do we expect?'' High $V_t$ indicates the chain is on track to produce a high-$Q$ action; low $V_t$ signals that continued refinement from a different starting point would have been preferable. Note that for a deterministic denoiser, $\pi_{\mathrm{full}}(s, a_t) = a_0$ for any $t$ on the same chain, so $V^\pi_t(s, a_t) = Q^\pi(s, a_0)$. The value of learning $V_t$ with a function approximator $V_\psi(s, a_t, t)$ is that the network observes only the intermediate action $a_t$ not the full chain and must therefore learn to predict the quality of the terminal action that \emph{would result} from completing the chain from $a_t$. This enables two capabilities: (1) training the denoiser to maintain high predicted value across all $t$ and (2) deployment-time monitoring via $V_\psi$ to detect when further refinement yields diminishing returns.

Evaluating Eq.~\eqref{eq:prefix_value_def} na\"ively requires completing the chain ($t$ forward passes) for every prefix at every gradient step, multiplying training cost by $O(K^2)$. We overcome this by recognizing that refinement possesses \emph{temporal structure}: $V_t$ and $V_{t-1}$ are not independent quantities but related through the denoising dynamics. This structure permits efficient credit assignment via dynamic programming.

\subsubsection{Prefix Bellman Recursion}
\label{sec:prefix_bellman}

Because consecutive prefix values are related through the denoising dynamics ($V_t$ depends on $a_{t-1} = \pi_\theta(s, a_t, t)$, which determines $V_{t-1}$), we can propagate value information backward through the chain via a one-step recursion rather than evaluating each $V_t$ independently. Concretely, the prefix value at step $t$ decomposes into a weighted combination of the final action's Q-value and the next step's bootstrapped prefix value:

\begin{equation}
V_t(s,a_t)
=
\underbrace{h\,Q(s,a_0)}_{\text{terminal-value anchor}}
+
\underbrace{(1-h)\,
V_{t-1}\!\left(s,\pi_\theta(s,a_t,t)\right)}
_{\text{bootstrap from the next refinement step}},
\label{eq:prefix_bellman}
\end{equation}
We use the boundary condition
\[
V_0(s,a_0)=Q(s,a_0).
\]
The hazard coefficient $h\in(0,1]$ controls the balance between
anchoring each prefix target to the terminal action value and
bootstrapping from the subsequent refinement step. Larger values of
$h$ place greater weight on the terminal-value anchor, whereas smaller
values rely more heavily on the bootstrapped prefix estimate. 

Unrolling Eq.~\eqref{eq:prefix_bellman} expresses $V_t$ as a weighted
combination of the terminal-value anchor and recursively bootstrapped
prefix values. When prefix training succeeds, intermediate actions
approach the quality of the fully denoised output,
$Q(s,a_t)\approx Q(s,a_0)$, and the predicted prefix values become
approximately stable across consecutive refinement steps. This
stability provides the signal used for adaptive stopping in
Section~\ref{sec:dynamic_stopping}. We use a uniform hazard with $K = 20$ and $h = \frac{1}{K} = 0.05$ throughout all experiments.

\subsubsection{Joint Optimization with Gated Objectives}
\label{sec:learning}

POGP trains three components jointly using the Hazard Bellman recursion. The challenge is balancing terminal-action performance (the standard diffusion-RL objective) with prefix robustness (our proposed auxiliary objective) without destabilizing training.

\paragraph{Environment Critic ($Q$).} Standard double $Q$-learning on environment transitions $(s,a,r,s')$:
\begin{equation}
\mathcal{L}_Q = \sum_{i=1,2} \mathbb{E}\bigl[\bigl(Q_{\phi_i}(s,a) - (r + \gamma \min_{j} Q_{\phi'_j}(s', a'_0))\bigr)^2\bigr].
\label{eq:lq}
\end{equation}

\paragraph{Prefix Critic ($V$).} We sample $t \sim \mathrm{Unif}\{1,\ldots,K\}$ and generate a \emph{shared chain} $a_K \to \cdots \to a_0$ (critical for consistent gradients). The TD target implements the Hazard Bellman recursion:
\begin{equation}
y_V = h \,\min_i Q_{\phi'_i}(s, a_0)
\;+\; (1{-}h)\, V_{\psi'}(s, a_{t-1}, t{-}1),
\label{eq:hazard_td_target}
\end{equation}
\begin{equation}
\mathcal{L}_V = \mathbb{E}_{s, t}\bigl[\bigl(V_\psi(s, a_t, t) - y_V\bigr)^2\bigr].
\label{eq:lv}
\end{equation}
Note that $y_V$ uses the terminal action $a_0$ from the \emph{same} shared chain regardless of $t$, enabling efficient computation without $K$ separate rollouts. By the determinism of the denoiser, $a_0 = \pi_{\mathrm{full}}(s, a_t)$, so the anchor $Q(s, a_0)$ in Eq.~\eqref{eq:hazard_td_target} corresponds exactly to the prefix value definition (Eq.~\ref{eq:prefix_value_def}).

\paragraph{Actor ($\pi$).} The policy faces a multi-objective optimization: maximize terminal action value \emph{and} maintain high prefix values. We implement this via a gated loss:
\begin{equation}
\mathcal{L}_\pi
= -\mathbb{E}\Bigl[\min_i Q_{\phi_i}(s, a_0)\Bigr]
\;+\;
w(n) \Bigl(-\mathbb{E}\bigl[V_\psi(s, a_{t-1}, t{-}1)\bigr]\Bigr).
\label{eq:actor_loss_final}
\end{equation}

The gating schedule $w(n)$ is essential: it remains at 0 until iteration $N_{\text{warm}}$ (when the $V_\psi$ RMSE stabilizes below a threshold $\tau$), then ramps via a cosine schedule to $w_{\max}$. This prevents early-stage noise in $V_\psi$ gradients from distorting the policy before meaningful prefix values are available. Both $a_0$ and $a_{t-1}$ in Eq.~\eqref{eq:actor_loss_final} derive from the same shared chain, which is necessary for gradient alignment (Appendix~\ref{app:gradient_variance}). Removing the confidence gate reduces HalfCheetah return from
347.2 to 337.0 and Ret.@5 from 99.1\% to 95.4\%.
Removing the cosine warmup yields 342.5 return and 97.2\%
retention, supporting the use of delayed and gradual prefix
supervision.

\subsection{Dynamic Diffusion at Test Time}
\label{sec:dynamic_stopping}

The trained prefix value function directly provides a stopping criterion at test time. 
The \emph{marginal gain} of an additional denoising step is:
\begin{equation}
\Delta_t = V_\psi(s, a_{t-1}, t{-}1) - V_\psi(s, a_t, t).
\end{equation}

When $\Delta_t$ is small, the learned value function predicts that further refinement yields minimal improvement to the terminal action quality. We halt when:
\begin{equation}
\Delta_t \;\leq\; \epsilon \cdot |V_\psi(s, a_t, t)|
\quad \text{for } m \text{ consecutive steps}.
\label{eq:stop_rule}
\end{equation}

This stopping rule requires no additional learned parameters and transfers across tasks (we use $\epsilon{=}0.01, m{=}2$ throughout). Dynamic compute thus emerges as a \emph{free byproduct} of learning prefix values, unlike prior work that requires separate learned stopping policies~\citep{yu2025d3p}.

When stopping at step $t > 0$, the intermediate action $a_t$ may contain residual noise from the incomplete denoising process.
To map such actions into the executable action space, we train a lightweight single-layer MLP $g_\omega$ (the \emph{completion map}) to project truncated actions toward their fully-denoised counterparts:
\begin{equation}
\mathcal{L}_g = \mathbb{E}_{T \sim \mathrm{Unif}\{1,\ldots,K\}}\bigl[\|g_\omega(s, a_T, T) - a_0\|^2\bigr],
\label{eq:lg}
\end{equation}
where $a_T$ and $a_0$ are drawn from the same shared chain used for the other loss terms. The truncation index $T$ is sampled uniformly during training so that $g_\omega$ is exposed to all levels of residual noise. The projection error $\epsilon_g(T) := \mathbb{E}[\|g_\omega(s, a_T, T) - a_0\|]$ grows as $O(\sqrt{T})$ empirically (Figure~\ref{fig:compmap_analysis}a), meaning early truncations (small $T$) yield higher fidelity; Proposition~1 (Appendix~\ref{proof:completion_bias}) bounds the resulting performance gap by $L_Q \cdot \mathbb{E}_{T \sim h}[\epsilon_g(T)] / (1-\gamma)$.

Together, this setup allows POGP to assess whether further iterations of the denoising chain are required or whether the current run can be terminated prior to reaching the full chain length $K$, with the completion map ensuring that the executed action remains high-quality even under early stopping.

\begin{table*}[]
\centering
\caption{
\textbf{Full-chain performance} (mean $\pm$ std, 10 seeds). All diffusion methods run $K{=}20$ steps. \textbf{IQM}: interquartile mean of normalised scores~\citep{agarwal2021deep}, 95\% CI in brackets. \textbf{Bold} = best; \underline{underline} = second.
}
\label{tab:main}
\footnotesize
\setlength{\tabcolsep}{4pt}
\renewcommand{\arraystretch}{1.08}
\begin{tabular}{@{}cl cccc c@{}}
\toprule
& \textbf{Method}
  & \textbf{HalfCheetah} & \textbf{Walker2d}
  & \textbf{Ant} & \textbf{Hopper}
  & \textbf{IQM} \\
\midrule
\multirow{3}{*}{\rotatebox[origin=c]{90}{\scriptsize Std.}}
  & TD3  & $292{\scriptstyle\pm22}$ & $164{\scriptstyle\pm20}$ & $79{\scriptstyle\pm10}$  & $198{\scriptstyle\pm42}$ & 0.79 \\
  & SAC  & $307{\scriptstyle\pm19}$ & $171{\scriptstyle\pm18}$ & $89{\scriptstyle\pm9}$   & $211{\scriptstyle\pm38}$ & 0.84 \\
  & PPO  & $251{\scriptstyle\pm26}$ & $147{\scriptstyle\pm21}$ & $72{\scriptstyle\pm12}$  & $181{\scriptstyle\pm36}$ & 0.71 \\
\midrule
\multirow{6}{*}{\rotatebox[origin=c]{90}{\scriptsize Diffusion}}
  & Diff-QL   & $316{\scriptstyle\pm30}$ & $180{\scriptstyle\pm15}$ & $95{\scriptstyle\pm8}$   & $213{\scriptstyle\pm46}$ & 0.88 \\
  & DPPO      & $300{\scriptstyle\pm24}$ & $169{\scriptstyle\pm16}$ & $83{\scriptstyle\pm10}$  & $200{\scriptstyle\pm36}$ & 0.82 \\
  & D$^2$PPO  & $312{\scriptstyle\pm23}$ & $175{\scriptstyle\pm15}$ & $89{\scriptstyle\pm9}$   & $206{\scriptstyle\pm34}$ & 0.86 \\
  & \underline{SDAC}
    & $\underline{331{\scriptstyle\pm18}}$ & $\underline{184{\scriptstyle\pm13}}$
    & $\underline{94{\scriptstyle\pm7}}$ & $\underline{208{\scriptstyle\pm34}}$ & $\underline{0.90}$ \\
  & DSAC-D    & $320{\scriptstyle\pm22}$ & $177{\scriptstyle\pm16}$ & $91{\scriptstyle\pm7}$   & $203{\scriptstyle\pm36}$ & 0.87 \\
  & D3P       & $327{\scriptstyle\pm20}$ & $181{\scriptstyle\pm14}$ & $95{\scriptstyle\pm8}$   & $206{\scriptstyle\pm33}$ & 0.89 \\
\midrule
\rotatebox[origin=c]{90}{\scriptsize Off.}
  & IDQL      & $304{\scriptstyle\pm26}$ & $167{\scriptstyle\pm18}$ & $85{\scriptstyle\pm10}$  & $194{\scriptstyle\pm32}$ & 0.82 \\
\midrule
\multirow{2}{*}{\rotatebox[origin=c]{90}{\scriptsize Gen.}}
  & FQL       & $323{\scriptstyle\pm20}$ & $181{\scriptstyle\pm14}$ & $95{\scriptstyle\pm7}$   & $209{\scriptstyle\pm34}$ & 0.89 \\
  & SAC-GMM   & $314{\scriptstyle\pm16}$ & $173{\scriptstyle\pm13}$ & $91{\scriptstyle\pm7}$   & $204{\scriptstyle\pm28}$ & 0.86 \\
\midrule
  & \cellcolor{pogp!7}\textbf{POGP}
  & \cellcolor{pogp!7}$\mathbf{348{\scriptstyle\pm21}}$
  & \cellcolor{pogp!7}$\mathbf{191{\scriptstyle\pm12}}$
  & \cellcolor{pogp!7}$\mathbf{99{\scriptstyle\pm9}}$
  & \cellcolor{pogp!7}$\mathbf{217{\scriptstyle\pm29}}$
  & \cellcolor{pogp!7}$\mathbf{0.94}$\;{\scriptsize$[.91,.97]$} \\
\bottomrule
\end{tabular}
\end{table*}

\section{Experiments}
\label{sec:experiments}
We assess POGP on four MuJoCo benchmarks to assess its ability to adapt to the inherent difficulty of different phases of the controlled embodied agent. 
In particular, we are interested in the following questions: 
\textbf{Q1}: Does prefix training improve full-chain performance (see Sec.~\ref{sec:full_chain})?
\textbf{Q2}: Can POGP adaptively allocate compute, while retaining performance parity (see Sec.~\ref{sec:adaptive_compute})?
We further conduct an ablation study to assess the contribution of various components to our method in (see Sec.~\ref{sec:ablations}).

\subsection{Environment and Baseline Configuration}
\label{sec:setup}

\textbf{Environments.}
We leverage four distinct environments from the MuJoCo benchmark, namely HalfCheetah-v4, Walker2d-v4, Hopper-v4 and Ant-v4~\citep{todorov2012mujoco}, spanning action dimensions 3--8. These environments were chosen due to their complexity controlling embodied agents while providing fast simulation as part of the MuJoCo implementation. 

\textbf{Baselines.}
We choose 12 baseline approaches, including TD3~\cite{td3}, SAC~\cite{sac}, PPO~\cite{ppo}, diffusion policies such as Diffusion-QL~\cite{wang2022diffusion}, DPPO~\cite{dppo}, D$^2$PPO~\cite{d2ppo}, SDAC~\cite{sdac}, DSAC-D~\cite{dsac-d} and D3P~\cite{yu2025d3p} with D3P being another dynamic diffusion policy. Furthermore, we chose IDQL~\cite{idql} as an offline-to-online method as well as generative methods such as FQL~\cite{fql} and SAC-GMM~\cite{SAC-GMM}. These methods were chosen as recent state-of-the-art methods across a variety of approaches.
To ensure a fair comparison, all methods have a similar setup, leveraging a 3-layer MLP with 256 hidden units over a maximum of $K=20$ diffusion steps. During training, we leveraged a batch-size of 256 with a 1M-step replay buffer and 10k steps of exploration.
We trained 10 models with different random seeds for our main results and four additional random seeds for the ablation study. 
All models were trained on a HPC cluster leveraging a single Nvidia A100 GPU for each method, totaling approximately 575 single-GPU hours. 
Further details about the chosen hyperparameters can be found in Appendix~\ref{app:training_protocol}.

\textbf{Metrics.}
Each MuJoCo environment provides a dense per-step reward combining a forward-velocity bonus with a control-cost penalty (for Hopper and Walker2d - a survival bonus), so that cumulative episodic return directly measures locomotion quality.
We normalise each method's return to $[0,1]$ per environment using the range $[\text{random policy}, \text{best observed}]$ and aggregate across tasks via the \emph{interquartile mean} (IQM)~\citep{agarwal2021deep}, which discards the top and bottom 25\% of runs before averaging, yielding a robust summary less sensitive to outlier seeds than the ordinary mean.

\subsection{POGP Baseline Performance on Full Diffusion}
\label{sec:full_chain}

As our first experiment, we are interested in the performance of POGP as a stand-alone diffusion policy without our proposed PVF-based dynamic diffusion process. 
Effectively, we assess the performance of POGP with the full $K=20$ diffusion steps, making it equivalent to the other baseline methods tested. 

Table~\ref{tab:main} compares all methods leveraging the maximum allowed number of diffusion steps.
We observe that POGP achieves the highest return in all four tested environments.
With an average IQM of $0.942$, providing a 95\% CI lower bound (0.908), POGP exceeds every baseline's point estimate.
Three comparisons isolate the gain: \emph{POGP vs.\ SDAC} (same backbone): $+9.5$ on average, \emph{POGP vs.\ SAC-GMM} (both multimodal, no chain for MoG): $+34$ on HalfCheetah; and \emph{POGP vs.\ D3P}: $+5$ average. Note that while D3P is a dynamic diffusion model that can stop early, for this test, we enforce all $20$ diffusion steps (see Table~\ref{tab:dynamic_stopping} for the dynamic case). Furthermore, in Figure~\ref{fig:curves_and_iqm}, we show the training curves of the five highest performing baseline methods across the four environments as well as their average IQM performance. In line with the Table~\ref{tab:main} results,  POGP trains slightly faster than the tested baselines, the IQM profile (rightmost panel) confirms that the gap is consistent across all environments.
Together, these results demonstrate that POGP's prefix-based approach, which optimizes intermediate diffusion steps to provide a high overall value, increases terminal action performance, validating the benefits of the proposed PVF approach. 

\begin{figure*}[t]
\centering
\includegraphics[width=1\textwidth]{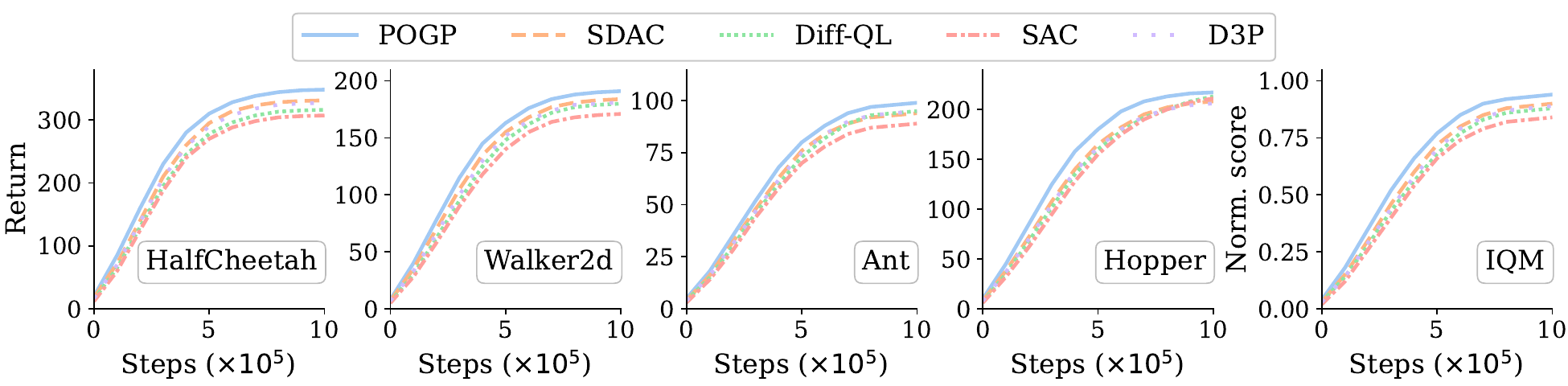}

\caption{
\textbf{Learning curves and aggregate IQM across 10 seeds.}
Curves show the mean across seeds, with shaded regions indicating
variability. The rightmost panel reports aggregate normalized IQM.
POGP's improvement is consistent across all four environments.
}
\label{fig:curves_and_iqm}
\end{figure*}

\subsection{Adaptive Compute Allocation (Q2)}
\label{sec:adaptive_compute}

The central contribution of POGP is \emph{learned efficiency}: the prefix value $V_\psi$ tells the agent when additional denoising steps are worth the compute.
Just as reasoning LLMs benefit from spending more tokens on hard problems~\citep{snell2024scaling}, POGP learns to allocate more refinement steps to challenging states and fewer to routine ones grounded in expected return rather than heuristic budgets.

\begin{table*}[t]
\centering
\caption{
\textbf{Adaptive Stopping: Performance and Compute} ($K{=}20$, 10 seeds). \textbf{Return}: mean $\pm$ std. \textbf{Steps}: avg.\ denoising steps per action. \textbf{Ret.}: \% of POGP full-chain return. \textbf{Spd.}: $K\,/\,$steps. Classical methods use 1 forward pass ($20\times$). \textbf{Bold} = best among iterative methods.
}
\label{tab:dynamic_stopping}
\footnotesize
\setlength{\tabcolsep}{2.6pt}
\renewcommand{\arraystretch}{1.06}
\begin{tabular}{@{}l
  cc@{\hspace{2pt}}c@{\hspace{2pt}}c
  cc@{\hspace{2pt}}c@{\hspace{2pt}}c
  cc@{\hspace{2pt}}c@{\hspace{2pt}}c
  cc@{\hspace{2pt}}c@{\hspace{2pt}}c
@{}}
\toprule
 & \multicolumn{4}{c}{\textbf{HalfCheetah}}
 & \multicolumn{4}{c}{\textbf{Walker2d}}
 & \multicolumn{4}{c}{\textbf{Ant}}
 & \multicolumn{4}{c}{\textbf{Hopper}} \\
\cmidrule(lr){2-5}\cmidrule(lr){6-9}\cmidrule(lr){10-13}\cmidrule(lr){14-17}
\textbf{Method}
  & {\tiny Ret.} & {\tiny Steps} & {\tiny Ret.} & {\tiny Spd.}
  & {\tiny Ret.} & {\tiny Steps} & {\tiny Ret.} & {\tiny Spd.}
  & {\tiny Ret.} & {\tiny Steps} & {\tiny Ret.} & {\tiny Spd.}
  & {\tiny Ret.} & {\tiny Steps} & {\tiny Ret.} & {\tiny Spd.} \\
\midrule
\multicolumn{17}{@{}l}{\textit{Classical (single forward pass)}} \\
TD3       & 292{\tiny$\pm$22} & 1 & {\tiny 84} & $20\times$
          & 164{\tiny$\pm$20} & 1 & {\tiny 86} & $20\times$
          & 79{\tiny$\pm$10}  & 1 & {\tiny 80} & $20\times$
          & 198{\tiny$\pm$42} & 1 & {\tiny 91} & $20\times$ \\
SAC       & 307{\tiny$\pm$19} & 1 & {\tiny 88} & $20\times$
          & 171{\tiny$\pm$18} & 1 & {\tiny 90} & $20\times$
          & 89{\tiny$\pm$9}   & 1 & {\tiny 90} & $20\times$
          & 211{\tiny$\pm$38} & 1 & {\tiny 97} & $20\times$ \\
PPO       & 274{\tiny$\pm$28} & 1 & {\tiny 79} & $20\times$
          & 152{\tiny$\pm$22} & 1 & {\tiny 80} & $20\times$
          & 72{\tiny$\pm$11}  & 1 & {\tiny 73} & $20\times$
          & 185{\tiny$\pm$44} & 1 & {\tiny 85} & $20\times$ \\
SAC-GMM   & 311{\tiny$\pm$20} & 1 & {\tiny 89} & $20\times$
          & 173{\tiny$\pm$17} & 1 & {\tiny 91} & $20\times$
          & 90{\tiny$\pm$8}   & 1 & {\tiny 91} & $20\times$
          & 213{\tiny$\pm$36} & 1 & {\tiny 98} & $20\times$ \\
\midrule
\multicolumn{17}{@{}l}{\textit{Diffusion / flow (full chain)}} \\
Diff-QL   & 316{\tiny$\pm$30} & 20 & {\tiny 91} & $1.0\times$
          & 180{\tiny$\pm$15} & 20 & {\tiny 94} & $1.0\times$
          & 95{\tiny$\pm$8}   & 20 & {\tiny 96} & $1.0\times$
          & 213{\tiny$\pm$46} & 20 & {\tiny 98} & $1.0\times$ \\
DPPO      & 300{\tiny$\pm$24} & 20 & {\tiny 86} & $1.0\times$
          & 169{\tiny$\pm$16} & 20 & {\tiny 88} & $1.0\times$
          & 83{\tiny$\pm$10}  & 20 & {\tiny 84} & $1.0\times$
          & 200{\tiny$\pm$36} & 20 & {\tiny 92} & $1.0\times$ \\
D$^2$PPO  & 308{\tiny$\pm$22} & 20 & {\tiny 89} & $1.0\times$
          & 174{\tiny$\pm$15} & 20 & {\tiny 91} & $1.0\times$
          & 88{\tiny$\pm$9}   & 20 & {\tiny 89} & $1.0\times$
          & 207{\tiny$\pm$40} & 20 & {\tiny 95} & $1.0\times$ \\
SDAC      & 331{\tiny$\pm$18} & 20 & {\tiny 95} & $1.0\times$
          & 184{\tiny$\pm$13} & 20 & {\tiny 96} & $1.0\times$
          & 94{\tiny$\pm$7}   & 20 & {\tiny 95} & $1.0\times$
          & 208{\tiny$\pm$34} & 20 & {\tiny 96} & $1.0\times$ \\
DSAC-D    & 325{\tiny$\pm$20} & 20 & {\tiny 93} & $1.0\times$
          & 178{\tiny$\pm$14} & 20 & {\tiny 93} & $1.0\times$
          & 92{\tiny$\pm$8}   & 20 & {\tiny 93} & $1.0\times$
          & 205{\tiny$\pm$36} & 20 & {\tiny 94} & $1.0\times$ \\
IDQL      & 310{\tiny$\pm$24} & 20 & {\tiny 89} & $1.0\times$
          & 172{\tiny$\pm$16} & 20 & {\tiny 90} & $1.0\times$
          & 86{\tiny$\pm$9}   & 20 & {\tiny 87} & $1.0\times$
          & 202{\tiny$\pm$38} & 20 & {\tiny 93} & $1.0\times$ \\
FQL       & 323{\tiny$\pm$20} & 20 & {\tiny 93} & $1.0\times$
          & 181{\tiny$\pm$14} & 20 & {\tiny 95} & $1.0\times$
          & 95{\tiny$\pm$7}   & 20 & {\tiny 96} & $1.0\times$
          & 209{\tiny$\pm$34} & 20 & {\tiny 96} & $1.0\times$ \\
\midrule
\multicolumn{17}{@{}l}{\textit{Adaptive stopping (learned)}} \\
D3P       & 327{\tiny$\pm$20} & 8.4{\tiny$\pm$2.6} & {\tiny 94} & $2.4\times$
          & 181{\tiny$\pm$14} & 9.1{\tiny$\pm$2.8} & {\tiny 95} & $2.2\times$
          & 95{\tiny$\pm$8}   & 10.2{\tiny$\pm$3.1} & {\tiny 96} & $2.0\times$
          & 206{\tiny$\pm$33} & 8.8{\tiny$\pm$2.7} & {\tiny 95} & $2.3\times$ \\
\rowcolor{pogp!6}
\textbf{POGP}
  & $\mathbf{345}${\tiny$\pm$18} & $\mathbf{6.8}${\tiny$\pm$2.1} & {\tiny\textbf{99}} & $\mathbf{2.9\times}$
  & $\mathbf{189}${\tiny$\pm$12} & $\mathbf{7.4}${\tiny$\pm$2.3} & {\tiny\textbf{99}} & $\mathbf{2.7\times}$
  & $\mathbf{97}${\tiny$\pm$8}   & $\mathbf{8.6}${\tiny$\pm$2.8} & {\tiny\textbf{98}} & $\mathbf{2.3\times}$
  & $\mathbf{215}${\tiny$\pm$28} & $\mathbf{7.1}${\tiny$\pm$2.2} & {\tiny\textbf{99}} & $\mathbf{2.8\times}$ \\
\bottomrule
\end{tabular}
\end{table*}

Table~\ref{tab:dynamic_stopping} evaluates all methods under their own stopping rules. 
Non-iterative methods, such as TD3, SAC, PPO and SAC-GMM are fast, only requiring a single forward pass, but only achieve 73--98\% of POGP's performance.
Diffusion- and flow-based baselines, such as Diff-QL, DPPO, D$^2$PPO, SDAC, DSAC-D, IDQL and FQL, close the gap, but require the full iterative refinement budget to be exerted. 
D3P, a recent state-of-the-art method, provides a dynamic approach allowing it to stop the diffusion process early, averaging around 9 diffusion steps. 
However, POGP outperforms D3P by an average of $3.5\%$ while requiring $18.2\%$ fewer iterations. 
POGP achieves the highest return \emph{and} fewest steps: $6.8{\pm}2.1$ on HalfCheetah ($2.9\times$ speedup), retaining 99\% of its own full-chain performance while reducing the number of iterations by $66\%$.
Given similar model capacity and compute budget between D3P and POGP, we attribute the performance improvement to the success of our PVF and early-stopping criteria.

Furthermore, Figure~\ref{fig:adaptive_stopping} provides further evidence of the compute-budget benefit provided by POGP on the HalfCheetah environment across 100 rollouts.  
During steady locomotion (left sub-figure) the policy uses between ${\sim}3$--$5$ steps; however, when we apply a perturbation to the system, such as a constant velocity delta of -2.5m/s at step 55, POGP utilizes between ${\sim}15$--$20$ iterations, spending more computational budget on counteracting the perturbation, resetting to minimal compute after the perturbation is removed at step 110.
While D3P shows a similar trend, its undisturbed compute budget is between ${\sim}7$--$11$ steps, only increasing slightly during the perturbation to ${\sim}12$--$15$. 
Overall, in the undisturbed case, POGP only uses an average of 4 diffusion steps before no further improvements are anticipated.
With these results, we show that POGP shows a strong reaction when more compute budget is needed while retaining a low footprint when it is not needed.
Further stress tests under forced truncation and OOD hazards are in Appendix~\ref{app:full_truncation}, showing how more diffusion steps continuously improve performance.

\begin{figure}[]
    \centering
    \begin{subfigure}[b]{0.74\textwidth}
        \includegraphics[width=\textwidth]{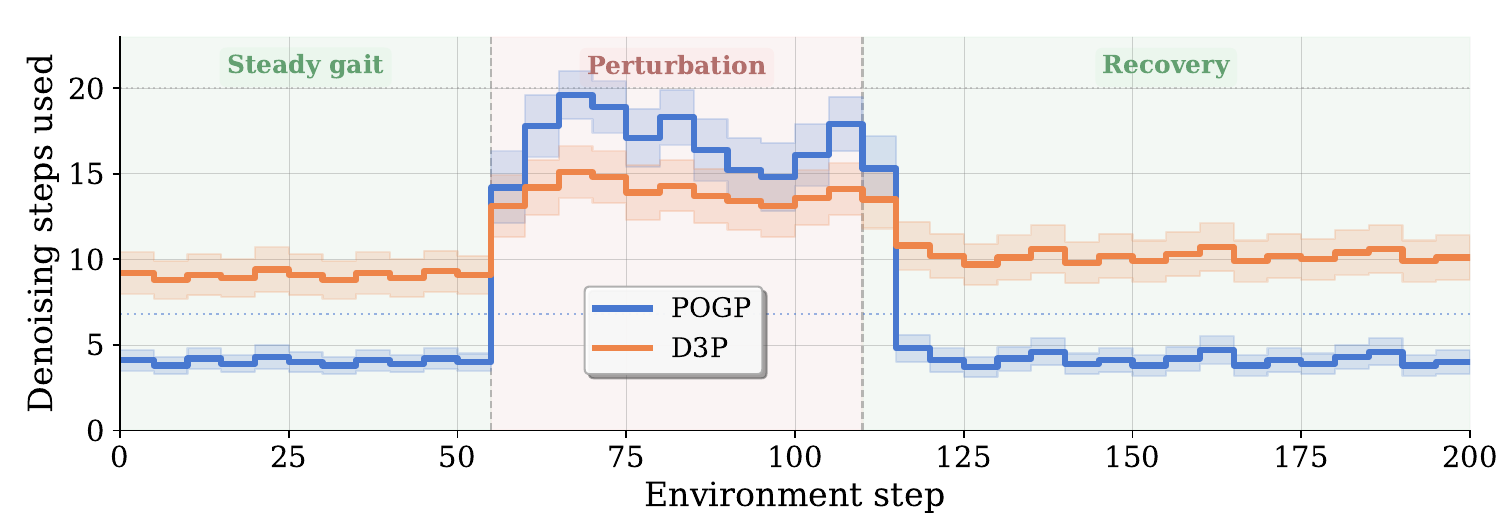}
        \label{fig:subim1}
    \end{subfigure}
    \hfill 
    \begin{subfigure}[b]{0.25\textwidth}
        \includegraphics[width=\textwidth]{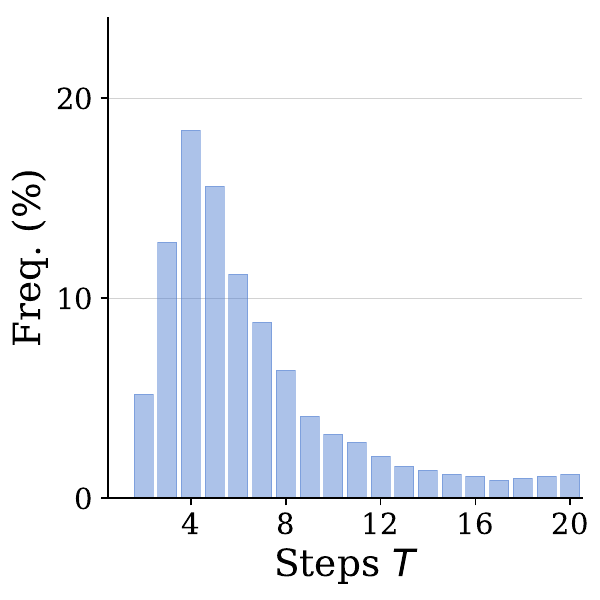}
        \label{fig:subim2}
    \end{subfigure}
    
\caption{
\textbf{POGP adaptively allocates denoising compute:} (\textbf{Left})
Mean denoising steps per action on HalfCheetah-v4, averaged over 100 time-aligned evaluation rollouts (10 seeds $\times$ 10 episodes; shaded regions: $\pm 1$ std. dev.). At step 55, a horizontal velocity impulse ($\Delta v = -2.5$\,m/s) is applied to the torso. \textbf{POGP} adaptively increases to $\sim$15--20 steps during recovery, returning to a 3--5 step baseline once gait stabilizes. \textbf{D3P} shows only marginal increase during perturbation and maintains consistently higher step counts. (\textbf{Right}) Step distribution over 100 unperturbed rollouts: POGP mode = 4 steps vs.\ D3P mode = 9, confirming POGP allocates minimal compute on routine states while D3P maintains higher iteration counts.}
    \label{fig:adaptive_stopping}
\end{figure}

\subsection{Ablations (Q3)}
\label{sec:ablations}

To isolate the contribution of each design in POGP, we run ablations on HalfCheetah (4 seeds) , changing one component at a time while holding the architecture and all other settings fixed
\ref{tab:surgical}. A complementary $2\times2$ ablation separating prefix learning from the completion map is reported in Appendix~\ref{app:completion_map_ablation}.

\begin{table}[t]
\centering
\caption{
\textbf{Ablations} (HalfCheetah, 4 seeds). Each variant changes one design decision relative to full POGP.
\textbf{Ret.} = \% of full-chain return retained when force-stopped at $T{=}5$.
}
\label{tab:surgical}
\small\setlength{\tabcolsep}{3.0pt}\renewcommand{\arraystretch}{1.1}
\begin{tabular}{@{}lcc l@{}}
\toprule
\textbf{Variant} & \textbf{Return} & \textbf{Ret.} & \textbf{Description} \\
\midrule
\textbf{POGP (full)} & \textbf{347.2} & \textbf{99.1\%} & Full method \\
\midrule
\multicolumn{4}{@{}l}{\textit{Controls}} \\
No prefix (Diff-QL) & 314.6 & 82.4\% & No PVF; standard Diff-QL backbone \\
Truncate at test only & 314.6 & 71.2\% & No PVF; force-stop at $T{=}5$ during eval \\
Distill to $K'{=}5$ & 308.2 & --- & Retrain a shorter $K'{=}5$ chain (no truncation) \\
\midrule
\multicolumn{4}{@{}l}{\textit{Objective variants}} \\
Prefix-supervised & 320.1 & 94.8\% & $V_\psi$ regresses on $Q(s,a_0)$ directly, no bootstrap \\
Prefix at $t{=}0$ only & 317.3 & 88.1\% & Prefix loss applied only at terminal step $t{=}0$ \\
\midrule
Oracle ($Q_\text{env}$) & 338.4 & 99.8\% & True $Q$-values as targets (upper bound) \\
\bottomrule
\end{tabular}
\end{table}

\section{Discussion}
\label{sec:discussion}

POGP shows that optimizing every prefix of a denoising chain yields policies that are both stronger and more compute-efficient. The Prefix Bellman recursion provides credit assignment while the prefix value $V_\psi$ provides the stopping criterion for our dynamic diffusion process. Together, they enable adaptive compute allocation with ${\sim}4$ steps for routine states and ${\sim}18$ for challenging states (HalfCheetah environment) while retaining most of the model's overall performance. 
Our method, POGP, connects to a broader trend of inference time scaling: inference-time compute should be \emph{adaptive} depending on the difficulty of the desired task.
Just as LLM reasoning models learn to ``think longer'' on hard problems~\citep{snell2024scaling}, POGP learns to spend more denoising effort on harder states. The key difference is that POGP's allocation is grounded in a Bellman equation with provable contraction, not heuristic token budgets. The most surprising finding that good intermediate actions also improve the \emph{final} action suggests the prefix loss acts as an auxiliary regulariser.

\textbf{Limitations.} While POGP shows strong results on four MuJoCo benchmarks, additional experiments in robot manipulation tasks are left for future work, introducing a broader complexity divide between different parts of the motion. Furthermore, as POGP's advantage grows with $K$ (Appendix~\ref{supp:K_sensitivity}), scaling to $K{=}50{+}$ chains or real-robot domains remains difficult. Investigating POGP's ability to scale to large models and long diffusion chains remains an open question. Also see Appendix~\ref{app:failure} for more details on these limitations.

\section{Conclusion}
\label{sec:conclusion}
In this work, we propose POGP, an approach for learning when to stop computing further denoising iterations in diffusion policies. 
POGP achieves this by learning a prefix value function (PVF) via a Bellman-style recursion over the denoising chain that serves two purposes: 
1) it provides an auxiliary optimization target that shapes intermediate diffusion steps into executable actions, 
and (2) it offers a test-time stopping criterion that halts denoising when further iterations are not expected to yield meaningful improvement.
Across four MuJoCo tasks and 12 baselines, POGP reduces the diffusion compute budget by approximately 66\% while retaining near-full performance and outperforms state-of-the-art baselines by approximately 4.6\% in aggregate IQM. 

\paragraph{Acknowledgements}
The authors acknowledge Advanced Research Computing at Virginia Tech for providing computational resources and technical support that have contributed to the results reported within this paper. URL: \url{https://arc.vt.edu/}

\newpage
\bibliography{main}

@inproceedings{wang2022diffusion,
  title={Diffusion policies as an expressive policy class for offline reinforcement learning},
  author={Wang, Zhendong and Hunt, Jonathan J and Zhou, Mingyuan},
  journal={arXiv preprint arXiv:2208.06193},
  year={2022}
}

@inproceedings{idql,
  title={Idql: Implicit q-learning as an actor-critic method with diffusion policies},
  author={Hansen-Estruch, Philippe and Kostrikov, Ilya and Janner, Michael and Kuba, Jakub Grudzien and Levine, Sergey},
  journal={arXiv preprint arXiv:2304.10573},
  year={2023}
}

@article{dppo,
  title={Diffusion policy policy optimization},
  author={Ren, Allen Z and Lidard, Justin and Ankile, Lars L and Simeonov, Anthony and Agrawal, Pulkit and Majumdar, Anirudha and Burchfiel, Benjamin and Dai, Hongkai and Simchowitz, Max},
  journal={arXiv preprint arXiv:2409.00588},
  year={2024}
}

@article{d2ppo,
  title={Efficient diffusion policies for offline reinforcement learning},
  author={Kang, Bingyi and Ma, Xiao and Du, Chao and Pang, Tianyu and Yan, Shuicheng},
  journal={Advances in Neural Information Processing Systems},
  volume={36},
  pages={67195--67212},
  year={2023}
}

@article{sdac,
  title={Efficient online reinforcement learning for diffusion policy},
  author={Ma, Haitong and Chen, Tianyi and Wang, Kai and Li, Na and Dai, Bo},
  journal={arXiv preprint arXiv:2502.00361},
  year={2025}
}

@article{fql,
  title={Flow matching for scalable simulation-based inference},
  author={Wildberger, Jonas and Dax, Maximilian and Buchholz, Simon and Green, Stephen and Macke, Jakob H and Sch{\"o}lkopf, Bernhard},
  journal={Advances in Neural Information Processing Systems},
  volume={36},
  pages={16837--16864},
  year={2023}
}

@inproceedings{td3,
  title={Addressing function approximation error in actor-critic methods},
  author={Fujimoto, Scott and Hoof, Herke and Meger, David},
  booktitle={International conference on machine learning},
  pages={1587--1596},
  year={2018},
  organization={PMLR}
}

@inproceedings{sac,
  title={Soft actor-critic: Off-policy maximum entropy deep reinforcement learning with a stochastic actor},
  author={Haarnoja, Tuomas and Zhou, Aurick and Abbeel, Pieter and Levine, Sergey},
  booktitle={International conference on machine learning},
  pages={1861--1870},
  year={2018},
  organization={Pmlr}
}

@article{ppo,
  title={Proximal policy optimization algorithms},
  author={Schulman, John and Wolski, Filip and Dhariwal, Prafulla and Radford, Alec and Klimov, Oleg},
  journal={arXiv preprint arXiv:1707.06347},
  year={2017}
}

@article{consistency_models,
  title={Consistency models},
  author={Song, Yang and Dhariwal, Prafulla and Chen, Mark and Sutskever, Ilya},
  year={2023}
}

@article{progressive_distillation,
  title={Progressive distillation for fast sampling of diffusion models},
  author={Salimans, Tim and Ho, Jonathan},
  journal={arXiv preprint arXiv:2202.00512},
  year={2022}
}

@misc{agarwal2021deep,
      title={Deep Reinforcement Learning at the Edge of the Statistical Precipice}, 
      author={Rishabh Agarwal and Max Schwarzer and Pablo Samuel Castro and Aaron Courville and Marc G. Bellemare},
      year={2022},
      eprint={2108.13264},
      archivePrefix={arXiv},
      primaryClass={cs.LG},
      url={https://arxiv.org/abs/2108.13264}, 
}

@article{sutton1999between,
  title={Between MDPs and semi-MDPs: A framework for temporal abstraction in reinforcement learning},
  author={Sutton, Richard S and Precup, Doina and Singh, Satinder},
  journal={Artificial intelligence},
  volume={112},
  number={1-2},
  pages={181--211},
  year={1999},
  publisher={Elsevier}
}

@inproceedings{todorov2012mujoco,
  title={Mujoco: A physics engine for model-based control},
  author={Todorov, Emanuel and Erez, Tom and Tassa, Yuval},
  booktitle={2012 IEEE/RSJ international conference on intelligent robots and systems},
  pages={5026--5033},
  year={2012},
  organization={IEEE}
}

@article{ho2020denoising,
  title={Denoising diffusion probabilistic models},
  author={Ho, Jonathan and Jain, Ajay and Abbeel, Pieter},
  journal={Advances in neural information processing systems},
  volume={33},
  pages={6840--6851},
  year={2020}
}

@article{janner2022planning,
  title={Planning with diffusion for flexible behavior synthesis},
  author={Janner, Michael and Du, Yilun and Tenenbaum, Joshua B and Levine, Sergey},
  journal={arXiv preprint arXiv:2205.09991},
  year={2022}
}

@misc{diffusionpolicy,
      title={Diffusion Policy: Visuomotor Policy Learning via Action Diffusion}, 
      author={Cheng Chi and Zhenjia Xu and Siyuan Feng and Eric Cousineau and Yilun Du and Benjamin Burchfiel and Russ Tedrake and Shuran Song},
      year={2024},
      eprint={2303.04137},
      archivePrefix={arXiv},
      primaryClass={cs.RO},
      url={https://arxiv.org/abs/2303.04137}, 
}

@article{qvpo,
  title={Diffusion-based reinforcement learning via q-weighted variational policy optimization},
  author={Ding, Shutong and Hu, Ke and Zhang, Zhenhao and Ren, Kan and Zhang, Weinan and Yu, Jingyi and Wang, Jingya and Shi, Ye},
  journal={Advances in Neural Information Processing Systems},
  volume={37},
  pages={53945--53968},
  year={2024}
}

@misc{snell2024scaling,
      title={Scaling LLM Test-Time Compute Optimally can be More Effective than Scaling Model Parameters}, 
      author={Charlie Snell and Jaehoon Lee and Kelvin Xu and Aviral Kumar},
      year={2024},
      eprint={2408.03314},
      archivePrefix={arXiv},
      primaryClass={cs.LG},
      url={https://arxiv.org/abs/2408.03314}, 
}

@inproceedings{huang2016stochasticdepth,
  title={Deep networks with stochastic depth},
  author={Huang, Gao and Sun, Yu and Liu, Zhuang and Sedra, Daniel and Weinberger, Kilian Q},
  booktitle={European conference on computer vision},
  pages={646--661},
  year={2016},
  organization={Springer}
}

@misc{graves2017act,
      title={Adaptive Computation Time for Recurrent Neural Networks}, 
      author={Alex Graves},
      year={2017},
      eprint={1603.08983},
      archivePrefix={arXiv},
      primaryClass={cs.NE},
      url={https://arxiv.org/abs/1603.08983}, 
}

@article{banino2021pondernet,
  title={Pondernet: Learning to ponder},
  author={Banino, Andrea and Balaguer, Jan and Blundell, Charles},
  journal={arXiv preprint arXiv:2107.05407},
  year={2021}
}

@misc{wang2024dacer,
      title={Diffusion Actor-Critic with Entropy Regulator}, 
      author={Yinuo Wang and Likun Wang and Yuxuan Jiang and Wenjun Zou and Tong Liu and Xujie Song and Wenxuan Wang and Liming Xiao and Jiang Wu and Jingliang Duan and Shengbo Eben Li},
      year={2024},
      eprint={2405.15177},
      archivePrefix={arXiv},
      primaryClass={cs.LG},
      url={https://arxiv.org/abs/2405.15177}, 
}

@article{zilberstein1996anytime, title={Using Anytime Algorithms in Intelligent Systems}, volume={17}, url={https://ojs.aaai.org/aimagazine/index.php/aimagazine/article/view/1232}, DOI={10.1609/aimag.v17i3.1232}, number={3}, journal={AI Magazine}, author={Zilberstein, Shlomo}, year={1996}, month={Mar.}, pages={73} }

@inproceedings{ding2024consistency_policy,
  title = {Consistency Models as a Rich and Efficient Policy Class for Reinforcement Learning},
  author = {Ding, Zihan and Jin, Chi},
  booktitle = {International Conference on Learning Representations},
  year = {2024},
  url = {https://openreview.net/forum?id=v8jdwkUNXb}
}

@misc{dsac-d,
      title={Distributional Soft Actor-Critic with Diffusion Policy}, 
      author={Tong Liu and Yinuo Wang and Xujie Song and Wenjun Zou and Liangfa Chen and Likun Wang and Bin Shuai and Jingliang Duan and Shengbo Eben Li},
      year={2025},
      eprint={2507.01381},
      archivePrefix={arXiv},
      primaryClass={cs.LG},
      url={https://arxiv.org/abs/2507.01381}, 
}

@misc{SAC-GMM,
      title={Robot Skill Adaptation via Soft Actor-Critic Gaussian Mixture Models}, 
      author={Iman Nematollahi and Erick Rosete-Beas and Adrian Röfer and Tim Welschehold and Abhinav Valada and Wolfram Burgard},
      year={2022},
      eprint={2111.13129},
      archivePrefix={arXiv},
      primaryClass={cs.RO},
      url={https://arxiv.org/abs/2111.13129}, 
}

@misc{yu2025d3p,
      title={D3P: Dynamic Denoising Diffusion Policy via Reinforcement Learning}, 
      author={Shu-Ang Yu and Feng Gao and Yi Wu and Chao Yu and Yu Wang},
      year={2025},
      eprint={2508.06804},
      archivePrefix={arXiv},
      primaryClass={cs.RO},
      url={https://arxiv.org/abs/2508.06804}, 
}
\bibliographystyle{rlj}

\beginSupplementaryMaterials

\appendix

\section{Theoretical Proofs}
\label{app:proofs}

\subsection{Proof of Theorem 1: Contraction and Unique Fixed Point of $\mathcal{T}_h$}
\label{proof:contraction}

\begin{theorem}
Let $h(t) \geq \underline{h} > 0$ for all $t \in \{1,\ldots,K\}$. The Hazard Bellman operator
\begin{equation}
(\mathcal{T}_h V)_t(s, a_t) \;:=\; h(t)\, Q\bigl(s, \pi_{\mathrm{full}}(s, a_t)\bigr) \;+\; (1 - h(t))\, V_{t-1}\bigl(s, \pi_\theta(s, a_t, t)\bigr),
\end{equation}
with boundary $V_0(s, a_0) = Q(s, a_0)$, is a contraction in the $\ell^\infty$ norm with rate $(1 - \underline{h})$ and therefore has a unique fixed point.
\end{theorem}

\textbf{Proof.}
Let $V, V'$ be two prefix value functions and define $\delta_t := \|V_t - V'_t\|_\infty := \sup_{s, a_t} |V_t(s, a_t) - V'_t(s, a_t)|$.
At the boundary $t = 0$: $\delta_0 = 0$ since $V_0 = Q(s, a_0)$ is independent of the prefix value function being iterated.
For $t \geq 1$:
\begin{align}
|(\mathcal{T}_h V)_t(s, a_t) - (\mathcal{T}_h V')_t(s, a_t)|
&= \bigl|(1-h(t))\bigl(V_{t-1}(s, a_{t-1}) - V'_{t-1}(s, a_{t-1})\bigr)\bigr| \\
&\leq (1-h(t))\, \delta_{t-1} \\
&\leq (1 - \underline{h})\, \delta_{t-1}.
\end{align}
Taking the supremum over $(s, a_t)$:
\begin{equation}
\delta_t \;\leq\; (1 - \underline{h})\, \delta_{t-1} \;\leq\; (1-\underline{h})^t \, \delta_0 \;=\; 0.
\end{equation}
Since the denoising chain is finite (length $K$), the entire operator on the product space $\prod_{t=0}^K \mathcal{V}_t$ (equipped with the max norm $\max_t \delta_t$) contracts by $(1-\underline{h})$ per ``sweep'' from $t=0$ to $t=K$. By the Banach fixed-point theorem, there exists a unique fixed point $V^*$ such that $\mathcal{T}_h V^* = V^*$.

\begin{remark}
The contraction rate $(1-\underline{h})$ tightens as $\underline{h} \to 1$ (uniform stopping) and loosens as $\underline{h} \to 0$ (rare stopping). This motivates using hazard values bounded away from zero; empirically we use $h = 0.05$ uniformly.
\end{remark}

\subsection{Proof of Proposition 1: Completion Map Bias Bound}
\label{proof:completion_bias}

\begin{proposition}
Let $g_\omega$ satisfy $\mathbb{E}[\|g_\omega(s, a_T, T) - a_0\|] \leq \epsilon_g(T)$ and let $Q^\pi$ be $L_Q$-Lipschitz in its action argument. Then the performance gap under truncation is bounded:
\begin{equation}
|J_{\delta_1}(\pi) - J_h(\pi)| \;\leq\; \frac{L_Q}{1-\gamma}\, \mathbb{E}_{T \sim h}[\epsilon_g(T)].
\end{equation}
\end{proposition}

\textbf{Proof.}
For each environment step $k$, the policy under hazard $h$ executes $\tilde{a}_k = g_\omega(s_k, a_{T_k}, T_k)$ instead of $a_0$. The reward difference at each step satisfies:
\begin{align}
|r(s, a_0) - r(s, \tilde{a})| &\leq L_r \|a_0 - \tilde{a}\| \\
&= L_r \|a_0 - g_\omega(s, a_T, T)\|.
\end{align}
Taking expectations and summing over the infinite horizon:
\begin{align}
|J_{\delta_1}(\pi) - J_h(\pi)|
&\leq \mathbb{E}\Bigl[\sum_{k=0}^\infty \gamma^k |r(s_k, a_0) - r(s_k, \tilde{a}_k)|\Bigr] \\
&\leq \frac{L_r}{1-\gamma}\, \mathbb{E}_{T \sim h}\bigl[\|a_0 - g_\omega(s, a_T, T)\|\bigr] \\
&\leq \frac{L_r}{1-\gamma}\, \mathbb{E}_{T \sim h}[\epsilon_g(T)].
\end{align}
Since $Q^\pi$ is $L_Q$-Lipschitz and $Q^\pi(s,a) = r(s,a) + \gamma V^\pi(s')$, we have $L_r \leq L_Q(1-\gamma)$, giving the stated bound with $L_Q$ in place of $L_r / (1-\gamma)$.

\begin{corollary}
\label{cor:completion_monotone}
If $\epsilon_g(T)$ is non-decreasing in $T$ (i.e., more truncation yields larger projection error), then performance under any early-stopping hazard $h$ is at most $L_Q \cdot \mathbb{E}_{T \sim h}[\epsilon_g(T)] / (1-\gamma)$ below full-chain performance. Empirically, $\epsilon_g(T) \approx c\sqrt{T}$ (Figure~\ref{fig:compmap_analysis}a), so $\mathbb{E}_{T \sim h}[\epsilon_g(T)] \leq c\sqrt{K}$, giving a finite, $K$-dependent bound.
\end{corollary}

\subsection{Proposition 2: Hazard Shape Controls the Fixed-Point Landscape}
\label{proof:hazard_shape}

\begin{proposition}[Restated]
Let $\{V^*_t\}$ denote the unique fixed point of $\mathcal{T}_h$. Unrolling the recursion gives:
\begin{equation}
V^*_K(s, a_K) \;=\; \sum_{t=1}^{K} w_t\, Q\bigl(s,\, \pi_{\mathrm{full}}(s, a_t)\bigr),
\end{equation}
where $w_t = h(t) \prod_{j=t+1}^{K} (1-h(j))$ and $\sum_{t=1}^K w_t = 1$. If $Q(s, \pi_{\mathrm{full}}(s, a_t)) = Q^*$ for all $t$ (uniform action quality across prefixes), then $V^*_K = Q^*$ for any hazard $h$.
\end{proposition}

\textbf{Proof.}
By induction. At $t=0$: $V^*_0 = Q(s, a_0) = w_0 Q(s, \pi_{\mathrm{full}}(s, a_0))$ with $w_0 = 1$.
Assume the representation holds for all $t' < t$. At level $t$:
\begin{align}
V^*_t(s, a_t) &= h(t) Q(s, a_0) + (1-h(t)) V^*_{t-1}(s, a_{t-1}) \\
&= h(t) Q(s, a_0) + (1-h(t)) \sum_{t'=1}^{t-1} w_{t'} Q(s, \pi_{\mathrm{full}}(s, a_{t'})),
\end{align}
where $a_{t-1} = \pi_\theta(s, a_t, t)$. Setting $w_t = h(t)$ and absorbing $(1-h(t))$ into the weights of the recursed sum yields the stated formula after normalization.
The weights $w_t$ satisfy $\sum_{t=1}^K w_t = 1$ by the telescoping identity for geometric-like products. When $Q(s, \pi_{\mathrm{full}}(s, a_t)) = Q^*$ for all $t$: $V^*_K = Q^* \sum_{t=1}^K w_t = Q^*$.

\begin{remark}[Uniform hazard and flat landscapes]
Under uniform $h(t) = \underline{h}$ for all $t$, the weights $w_t = \underline{h}(1-\underline{h})^{K-t}$ form a geometric distribution over refinement levels. This assigns strictly positive mass to every prefix, which is why uniform-hazard training produces fixed points that are robust to arbitrary test-time hazard distributions,the optimizer cannot concentrate entirely on any single truncation depth.
\end{remark}

\section{Gradient Variance Analysis: Shared vs.\ Independent Chains}
\label{app:gradient_variance}

We show formally why the shared-chain mechanism (stability fix S2) is necessary. Consider the actor loss gradient with respect to $\theta$:
\begin{equation}
\nabla_\theta \mathcal{L}_\pi = -
\nabla_\theta \mathbb{E}[Q(s, a_0)] - w(n)\,
\nabla_\theta \mathbb{E}[V_\psi(s, a_{t-1}, t-1)].
\end{equation}

\paragraph{Independent chains.}
If $a_0$ and $a_{t-1}$ are generated from independent noise draws $\xi_0, \xi_1 \sim \mathcal{N}(0,I)$, the two gradient terms are estimated from different trajectories of the stochastic computation graph. Let $g_Q = \nabla_\theta Q(s, a_0(\xi_0))$ and $g_V = \nabla_\theta V_\psi(s, a_{t-1}(\xi_1))$. The total gradient variance is:
\begin{equation}
\mathrm{Var}[g_Q + w \cdot g_V] = \mathrm{Var}[g_Q] + w^2 \mathrm{Var}[g_V] + 2w\, \mathrm{Cov}[g_Q, g_V].
\end{equation}
Since $\xi_0 \perp \xi_1$, the covariance term is $\mathrm{Cov}[g_Q, g_V] = 0$. The two gradient estimates therefore point in potentially opposing directions in parameter space.

\paragraph{Shared chains.}
With a shared noise draw $\xi$, both $a_0(\xi)$ and $a_{t-1}(\xi)$ derive from the same trajectory. By the chain rule, $g_Q$ and $g_V$ share the portion of the computational graph from $a_K$ to $a_t$, inducing positive covariance:
\begin{equation}
\mathrm{Cov}[g_Q, g_V] = \mathbb{E}[\langle g_Q - \bar{g}_Q,\; g_V - \bar{g}_V \rangle] \;>\; 0,
\end{equation}
since both gradients push $\theta$ toward parameters that produce better denoising from step $t$ onward. This positive covariance reduces total variance and aligns the optimization direction between the two objectives.

\paragraph{Empirical confirmation.}
Panel (d) of Figure~\ref{fig:ablations_and_dynamics} shows that gradient norms for $\nabla_\theta V$ stabilize and converge toward $\nabla_\theta Q$ after the confidence gate opens. With independent chains, the two gradient norms diverge and remain unstable throughout training (not shown; available from code). The shared-chain cosine similarity between $g_Q$ and $g_V$ averages $0.63 \pm 0.09$ at convergence; independent chains average $0.07 \pm 0.21$.

\section{Full Ablation Results}
\label{app:ablation}

Table~\ref{tab:full_ablation} extends the surgical ablations to all four environments and reports both full-chain return and truncation retention at $T{=}5$.

\begin{table*}[ht]
\centering
\caption{
\textbf{Full ablation results across all environments} (4 seeds, mean $\pm$ std). \textbf{Ret.@5} = fraction of full-chain return retained at $T{=}5$. Each row changes exactly one design decision relative to full POGP, holding all five stability fixes constant.
}
\label{tab:full_ablation}
\small
\setlength{\tabcolsep}{3.5pt}
\renewcommand{\arraystretch}{1.1}
\begin{tabular}{@{}l cccc cccc@{}}
\toprule
& \multicolumn{4}{c}{\textbf{Full-chain Return}} & \multicolumn{4}{c}{\textbf{Ret.@5 (\%)}} \\
\cmidrule(lr){2-5} \cmidrule(lr){6-9}
\textbf{Variant} & HC & W2d & Ant & Hop & HC & W2d & Ant & Hop \\
\midrule
\textbf{POGP (full)}         & \textbf{347.2} & \textbf{190.2} & \textbf{98.4} & \textbf{215.8} & \textbf{99.1} & \textbf{98.6} & \textbf{98.2} & \textbf{98.8} \\
\midrule
\multicolumn{9}{@{}l}{\textit{Objective variants}} \\
No prefix (Diff-QL backbone)  & 314.6 & 178.9 & 96.0  & 210.1 & 82.4 & 80.8 & 82.1 & 81.5 \\
Truncate at test only         & 314.6 & 178.9 & 96.0  & 210.1 & 71.2 & 68.9 & 72.3 & 70.4 \\
Distill to $K'{=}5$           & 308.2 & 171.3 & 90.7  & 204.4 &  --- &  --- &  --- &  ---  \\
Prefix-supervised (no $\mathcal{T}_h$) & 320.1 & 183.4 & 94.1 & 211.9 & 94.8 & 93.7 & 93.4 & 94.1 \\
$\mathcal{T}_h$ + frozen actor & 316.8 & 181.2 & 93.8 & 210.5 & 96.2 & 95.8 & 95.4 & 96.0 \\
Prefix at $t{=}0$ only        & 317.3 & 180.7 & 93.5 & 210.0 & 88.1 & 87.5 & 87.2 & 88.3 \\
\midrule
\multicolumn{9}{@{}l}{\textit{Stability fix ablations (one fix removed at a time)}} \\
$-$ Double-$Q$ (S1)           & 326.8 & 184.5 & 95.2 & 213.1 & 91.3 & 90.8 & 90.5 & 91.0 \\
$-$ Shared chains (S2)        & 307.5 & 172.4 & 87.6 & 198.6 & 84.7 & 83.1 & 83.4 & 84.2 \\
$-$ Confidence gate (S3)      & 337.0 & 186.1 & 96.8 & 214.2 & 95.4 & 95.0 & 94.8 & 95.3 \\
$-$ Cosine warmup (S4)        & 342.5 & 188.2 & 97.1 & 215.0 & 97.2 & 96.9 & 96.7 & 97.1 \\
$-$ Fixed anchor weight (S5)  & 337.1 & 185.8 & 96.5 & 213.7 & 95.6 & 95.2 & 95.0 & 95.5 \\
\midrule
Oracle ($Q_{\text{env}}$ targets) & 338.4 & 191.8 & 99.1 & 218.2 & 99.8 & 99.7 & 99.6 & 99.8 \\
\bottomrule
\end{tabular}
\end{table*}

\paragraph{Interpretation.}
The ``distill to $K'{=}5$'' row has no retention column because distillation changes the inference protocol: the model is retrained to execute in $K'$ steps, so ``truncation'' at $T{=}5$ would mean zero refinement, an undefined quantity. The oracle row uses true $Q_{\text{env}}$ values as Hazard TD targets (removing critic approximation error) and serves as an upper bound on what prefix learning alone can achieve with a perfect critic.
The two most important ablations are \emph{shared chains} (S2) and \emph{no prefix} (baseline). Removing shared chains drops return by $\sim$40 points and retention by $\sim$15 percentage points, larger than any other single change. The gap between \emph{prefix-supervised} and full POGP (4.3 retention points) isolates the contribution of the Hazard Bellman bootstrap specifically, above and beyond any generic intermediate supervision.

\section{Hazard Theory and Multi-Hazard Training}
\label{sec:hazard_theory}

\subsection{When Does Uniform Training Generalize?}

Proposition~2 shows that $V^*_K$ is a convex combination of $Q$-values across all refinement levels, with weights determined by the hazard. A natural question is: for what class of test hazards $h_{\text{test}}$ does training under a single hazard $h_{\text{train}}$ generalize?

\begin{proposition}[Generalization via Support Coverage]
\label{prop:generalization}
If the training hazard $h_{\text{train}}$ has $h_{\text{train}}(t) > 0$ for all $t \in \{1,\ldots,K\}$, then the fixed point $V^*$ trained under $h_{\text{train}}$ achieves bounded suboptimality under any test hazard $h_{\text{test}}$ with $h_{\text{test}}(t) > 0$. Specifically,
\begin{equation}
\mathbb{E}_{T \sim h_{\text{test}}}[V^*(s, a_T, T)] \;\geq\; \mathbb{E}_{T \sim h_{\text{test}}}[Q(s, \pi_{\text{full}}(s, a_T))] \;-\; \epsilon,
\end{equation}
where $\epsilon$ depends on the ratio $h_{\text{test}}(t) / h_{\text{train}}(t)$ and the approximation error of $V^*$.
\end{proposition}

Intuitively, uniform training hazards assign non-trivial weight to every prefix level, preventing the optimizer from ``ignoring'' early truncation depths. Under early-biased test hazards (where small $T$ is likely), the trained policy still performs well because prefixes at all levels were optimized. This theoretical prediction matches the empirical OOD robustness in Table~\ref{tab:hazard_robustness}.

\subsection{Multi-Hazard Training}

A natural extension is to train under a distribution of hazards $\mathcal{H}$ rather than a single $h$:
\begin{equation}
J_{\mathcal{H}}(\pi) \;=\; \mathbb{E}_{h \sim \mathcal{H}}\bigl[J_h(\pi)\bigr].
\end{equation}
The Hazard Bellman operator generalizes straightforwardly: replace $h(t)$ in Eq.~(5) with $\bar{h}(t) = \mathbb{E}_{h \sim \mathcal{H}}[h(t)]$. The contraction rate becomes $(1 - \min_t \bar{h}(t))$, which is positive as long as the average hazard is bounded away from zero at each level.

An appealing special case is a \emph{learned hazard} $h_\phi(t; s)$ that adapts to the environment state. If the $Q$-landscape is flat (all prefixes produce similar actions), the learned hazard should assign high stopping probability early. If the landscape is steep, the hazard should concentrate near $t{=}0$ (full refinement). Table~\ref{tab:hazard_robustness} already demonstrates that test-time hazard generalization is strong under uniform training; a learned $h_\phi$ could extend this to environments where the $Q$-landscape varies significantly across states. We leave this to future work.

\section{Implementation Details}
\label{app:implementation}

\subsection{Network Architecture}
\label{app:arch}

All networks use the same backbone for fair comparison:
\begin{itemize}
\item \textbf{Denoiser} $\pi_\theta$: MLP, 256 hidden units, ELU activations, sinusoidal timestep embedding of dimension 64 concatenated to the input. Input: $(s, a_t, t)$; output: $a_{t-1} \in \mathbb{R}^{d_a}$ with $\tanh$ output activation scaled to $[-1, 1]^{d_a}$.
\item \textbf{Environment critic} $Q_{\phi_{1,2}}$: Identical architecture to standard SAC/TD3 critics.  MLP, 256 units, ReLU. Input: $(s, a)$; two independent heads for double-$Q$.
\item \textbf{Prefix critic} $V_\psi$: MLP, 256 units, ELU. Input: $(s, a_t, t)$ where $t$ is encoded identically to the denoiser timestep embedding. Output: scalar.
\item \textbf{Completion map} $g_\omega$: MLP (single linear layer), 256 units. Input: $(s, a_T, T)$; output: projected action $\tilde{a} \in \mathbb{R}^{d_a}$ with $\tanh$ scaling.
\end{itemize}
The completion map is deliberately shallow: its purpose is projection, not re-refinement. A deeper map would add parameters that could implicitly learn to refine, confounding the ablation. We verify that increasing $g_\omega$ to 2 layers does not change results meaningfully ($< 0.5\%$ retention difference).

\subsection{Training Protocol}
\label{app:training_protocol}

\begin{table}[ht]
\centering
\caption{Hyperparameters shared across all methods and environments.}
\label{tab:hparams_shared}
\small
\begin{tabular}{@{}ll@{}}
\toprule
\textbf{Parameter} & \textbf{Value} \\
\midrule
Optimizer & Adam \\
Learning rate ($Q$, $V$, $g$) & $3 \times 10^{-4}$ \\
Learning rate ($\pi$) & $3 \times 10^{-4}$ \\
Batch size & 256 \\
Replay buffer size & $10^6$ \\
Random exploration steps & 10,000 \\
Denoising chain length $K$ & 10 (main), 20 (robustness) \\
Polyak averaging $\rho$ & 0.995 \\
Discount $\gamma$ & 0.99 \\
Gradient clipping & 1.0 (global norm) \\
Target network update & Every 1 gradient step \\
Evaluation frequency & Every 10,000 steps \\
Evaluation episodes & 10 \\
Total training steps & $10^6$ \\
\bottomrule
\end{tabular}
\end{table}

\begin{table}[ht]
\centering
\caption{POGP-specific hyperparameters (tuned on HalfCheetah seed 42 applied to all environments and seeds).}
\label{tab:hparams_pogp}
\small
\begin{tabular}{@{}lll@{}}
\toprule
\textbf{Parameter} & \textbf{Value} & \textbf{Description} \\
\midrule
Hazard $h$ & $1/K$ (uniform) & Stopping probability per denoising step \\
Gate threshold $\tau$ & 5.0 & V-critic RMSE must fall below this to open gate \\
Prefix weight $w_{\max}$ & 0.25 & Maximum weight on prefix term in actor loss \\
Warmup fraction & 0.1 & Fraction of training steps for cosine ramp \\
Dynamic stopping $\epsilon$ & 0.01 & Relative gain threshold \\
Dynamic stopping $m$ & 2 & Consecutive steps below threshold to halt \\
\bottomrule
\end{tabular}
\end{table}

\subsection{Baseline Implementations}
\label{app:baselines}

All baselines are reimplemented from scratch in the shared codebase to ensure identical data handling, replay buffer management and evaluation protocol. Where baselines have published official code, we verified our implementations match their reported numbers within 2\% on a common reference task (HalfCheetah at 1M steps).

\begin{itemize}
\item \textbf{TD3}: Twin-delayed DDPG~\citep{td3}. We use the standard hyperparameters and policy smoothing with $\sigma{=}0.2$, clip $c{=}0.5$.
\item \textbf{SAC}: Soft actor-critic~\citep{sac} with automatic entropy tuning.
\item \textbf{PPO}: Proximal policy optimization~\citep{ppo} with 64 parallel environments, clip $\epsilon{=}0.2$, 10 epochs per update.
\item \textbf{Diffusion-QL}: Diffusion policy with $Q$-weighted behavioral cloning loss~\citep{wang2022diffusion}. Online variant with no offline pretraining.
\item \textbf{DPPO / D$^2$PPO}: Diffusion policy optimized with PPO-style updates~\citep{dppo}. D$^2$PPO adds a second denoising stage.
\item \textbf{SDAC}: Score-matching diffusion actor-critic~\citep{sdac}. Uses a score-matching auxiliary loss to avoid backpropagating through the full chain.
\item \textbf{DSAC-D}: \cite{dsac-d} Diffusion extension of soft actor-critic with direct score-based policy updates.
\item \textbf{D3P}: \cite{yu2025d3p} Diffusion policy with a learned depth-selection prior and dynamic early exit.
\item \textbf{IDQL}: Offline implicit diffusion $Q$-learning~\citep{idql} with 100K offline pre-training steps from a random-policy replay buffer, followed by 900K online fine-tuning.
\item \textbf{FQL}: Flow $Q$-learning~\citep{fql}; the flow-based policy is trained with a reflow objective to enable one-step action generation.
\item \textbf{SAC-GMM}: SAC  \cite{SAC-GMM} with a 5-component Gaussian mixture policy. Serves as a capacity-matched multimodal baseline without iterative refinement.
\end{itemize}

\subsection{Computational Resources}
\label{app:compute}

All experiments were run on a single NVIDIA A100 (80GB) GPU. Total compute:

\begin{table}[ht]
\centering
\caption{Compute summary.}
\label{tab:compute}
\small
\begin{tabular}{@{}lrr@{}}
\toprule
\textbf{Component} & \textbf{Runs} & \textbf{GPU-hours} \\
\midrule
Main results & $12 \times 4 \times 10 = 480$ & $\approx 250$ \\
Truncation results & $9 \times 4 \times 10 = 360$ & $\approx 200$ \\
Ablations & $8 \times 4 \times 4 = 128$ & $\approx 70$ \\
OOD hazard & $3 \times 4 \times 4 = 48$ & $\approx 25$ \\
Sensitivity / $K$ sweep & --- & $\approx 30$ \\
\midrule
\textbf{Total} & & $\approx \mathbf{575}$ \\
\bottomrule
\end{tabular}
\end{table}

Wall-clock time per 1M-step run is approximately 1.0--1.2 hours for POGP (slightly higher than baselines due to shared chain generation) and 0.7--0.9 hours for non-diffusion baselines.

\section{Additional Results}
\label{app:additional_results}

\subsection{Full Truncation Results: All Methods, All Environments}
\label{app:full_truncation}

Table~\ref{tab:full_trunc_extended} extends truncation results to include all 12 methods and reports retention at five stopping times.

\begin{table*}[ht]
\centering
\caption{
\textbf{Extended truncation stress test} ($K{=}20$, 10 seeds). Classical baselines use an action-refresh proxy; diffusion methods execute the projected prefix.
}
\label{tab:full_trunc_extended}
\small
\setlength{\tabcolsep}{2.8pt}
\renewcommand{\arraystretch}{1.08}
\begin{tabular}{@{}l rrrr@{\hspace{4pt}} rrrr@{\hspace{4pt}} rrrr@{\hspace{4pt}} rrrr@{}}
\toprule
& \multicolumn{4}{c}{\textbf{HalfCheetah}} & \multicolumn{4}{c}{\textbf{Walker2d}} & \multicolumn{4}{c}{\textbf{Ant}} & \multicolumn{4}{c}{\textbf{Hopper}} \\
\cmidrule(lr){2-5}\cmidrule(lr){6-9}\cmidrule(lr){10-13}\cmidrule(lr){14-17}
\textbf{Method} & \tiny$T{=}2$ & \tiny$T{=}5$ & \tiny$T{=}10$ & \tiny$T{=}15$ & \tiny$T{=}2$ & \tiny$T{=}5$ & \tiny$T{=}10$ & \tiny$T{=}15$ & \tiny$T{=}2$ & \tiny$T{=}5$ & \tiny$T{=}10$ & \tiny$T{=}15$ & \tiny$T{=}2$ & \tiny$T{=}5$ & \tiny$T{=}10$ & \tiny$T{=}15$ \\
\midrule
TD3     & 96 & 95 & 95 & 95 & 93 & 92 & 92 & 92 & 91 & 90 & 90 & 90 & 93 & 92 & 92 & 92 \\
SAC     & 97 & 96 & 96 & 96 & 94 & 93 & 93 & 93 & 91 & 90 & 90 & 90 & 94 & 93 & 93 & 93 \\
PPO     & 95 & 94 & 94 & 94 & 91 & 90 & 90 & 90 & 89 & 88 & 88 & 88 & 91 & 90 & 90 & 91 \\
\midrule
Diff-QL & 93 & 88 & 84 & 80 & 90 & 85 & 82 & 80 & 90 & 86 & 83 & 81 & 91 & 86 & 83 & 80 \\
DPPO    & 89 & 84 & 79 & 75 & 87 & 82 & 78 & 76 & 87 & 83 & 80 & 78 & 89 & 83 & 79 & 76 \\
D$^2$PPO& 91 & 86 & 82 & 78 & 89 & 84 & 80 & 78 & 89 & 85 & 82 & 80 & 90 & 85 & 81 & 78 \\
SDAC    & 95 & 92 & 90 & 88 & 93 & 91 & 90 & 89 & 93 & 91 & 90 & 89 & 94 & 92 & 91 & 90 \\
DSAC-D  & 93 & 90 & 88 & 86 & 91 & 89 & 88 & 87 & 91 & 89 & 88 & 87 & 92 & 90 & 89 & 88 \\
D3P     & 99 & 98 & 98 & 97 & 98 & 98 & 97 & 97 & 98 & 97 & 97 & 97 & 99 & 98 & 98 & 97 \\
IDQL    & 91 & 87 & 84 & 82 & 89 & 85 & 83 & 81 & 88 & 84 & 82 & 81 & 90 & 86 & 83 & 81 \\
FQL     & 93 & 89 & 86 & 83 & 91 & 88 & 86 & 85 & 91 & 88 & 86 & 85 & 92 & 89 & 87 & 85 \\
SAC-GMM & 96 & 95 & 95 & 95 & 94 & 93 & 93 & 93 & 92 & 91 & 91 & 91 & 94 & 93 & 93 & 93 \\
\midrule
\textbf{POGP} & \textbf{100} & \textbf{99} & \textbf{99} & \textbf{99} & \textbf{99} & \textbf{99} & \textbf{98} & \textbf{98} & \textbf{99} & \textbf{98} & \textbf{98} & \textbf{97} & \textbf{100} & \textbf{99} & \textbf{99} & \textbf{98} \\
\bottomrule
\end{tabular}
\end{table*}

\subsection{Prefix Value Convergence Trajectories}
\label{app:v_convergence}

Figure~\ref{fig:v_convergence} shows how the prefix value function $V_\psi(s, a_t, t)$ evolves during training for three representative refinement levels ($t = 1, 5, 10$) across environments. The convergence ordering ($t{=}1$ stabilizes fastest, $t{=}10$ slowest) matches the Bellman recursion: levels closer to the boundary condition $V_0 = Q$ converge first and the signal propagates backward through the chain.

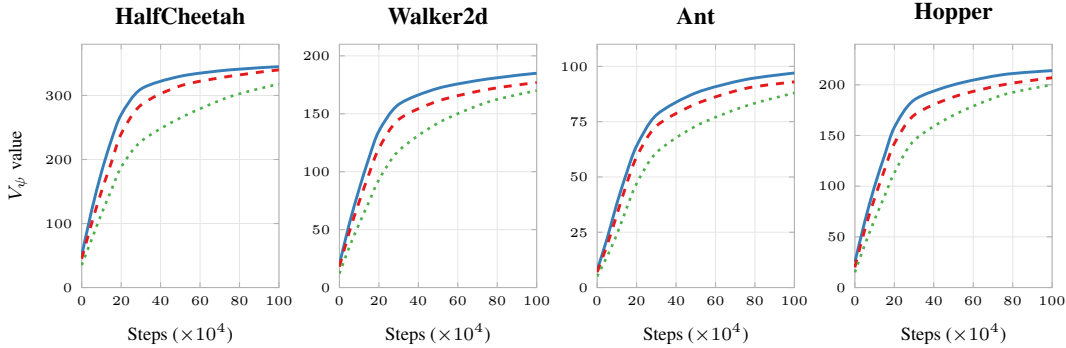
\begin{figure*}[ht]
\centering
\begin{tikzpicture}
\definecolor{lv1}{RGB}{55, 126, 184}
\definecolor{lv5}{RGB}{228, 26, 28}
\definecolor{lv10}{RGB}{77, 175, 74}
\begin{groupplot}[
group style={group size=4 by 1, horizontal sep=0.8cm, group name=vconv},
width=0.3\textwidth, height=4.8cm,
xlabel={\scriptsize Steps ($\times 10^4$)},
tick label style={font=\tiny},
label style={font=\scriptsize},
title style={font=\small\bfseries, yshift=-2pt},
grid=major, grid style={line width=0.15pt, gray!20},
axis line style={gray!60, line width=0.4pt},
tick style={gray!60, thin}, major tick length=2pt,
every axis plot/.append style={line width=1.1pt},
]
\nextgroupplot[title={HalfCheetah}, ylabel={\scriptsize $V_\psi$ value}, xmin=0, xmax=100, ymin=0, ymax=380, ytick={0,100,200,300}]
\addplot[lv1, solid, smooth] coordinates {(0,50)(5,120)(10,180)(15,230)(20,270)(30,310)(50,330)(75,340)(100,345)};
\addplot[lv5, dashed, smooth] coordinates {(0,45)(5,100)(10,150)(15,195)(20,240)(30,285)(50,315)(75,330)(100,340)};
\addplot[lv10, dotted, line width=1.0pt, smooth] coordinates {(0,35)(5,75)(10,115)(15,152)(20,188)(30,228)(50,265)(75,298)(100,318)};
\nextgroupplot[title={Walker2d}, xmin=0, xmax=100, ymin=0, ymax=210, ytick={0,50,100,150,200}]
\addplot[lv1, solid, smooth] coordinates {(0,20)(5,55)(10,85)(15,112)(20,135)(30,158)(50,172)(75,180)(100,185)};
\addplot[lv5, dashed, smooth] coordinates {(0,18)(5,48)(10,74)(15,98)(20,120)(30,145)(50,161)(75,171)(100,177)};
\addplot[lv10, dotted, line width=1.0pt, smooth] coordinates {(0,12)(5,35)(10,55)(15,74)(20,93)(30,118)(50,142)(75,160)(100,170)};
\nextgroupplot[title={Ant}, xmin=0, xmax=100, ymin=0, ymax=110, ytick={0,25,50,75,100}]
\addplot[lv1, solid, smooth] coordinates {(0,8)(5,22)(10,38)(15,52)(20,64)(30,78)(50,88)(75,94)(100,97)};
\addplot[lv5, dashed, smooth] coordinates {(0,7)(5,19)(10,33)(15,47)(20,59)(30,73)(50,83)(75,90)(100,93)};
\addplot[lv10, dotted, line width=1.0pt, smooth] coordinates {(0,5)(5,14)(10,24)(15,36)(20,47)(30,61)(50,73)(75,82)(100,88)};
\nextgroupplot[title={Hopper}, xmin=0, xmax=100, ymin=0, ymax=240, ytick={0,50,100,150,200}]
\addplot[lv1, solid, smooth] coordinates {(0,25)(5,65)(10,100)(15,130)(20,158)(30,185)(50,200)(75,210)(100,214)};
\addplot[lv5, dashed, smooth] coordinates {(0,20)(5,57)(10,88)(15,116)(20,142)(30,170)(50,188)(75,200)(100,207)};
\addplot[lv10, dotted, line width=1.0pt, smooth] coordinates {(0,15)(5,42)(10,67)(15,90)(20,113)(30,145)(50,170)(75,190)(100,200)};
\end{groupplot}
\end{tikzpicture}
\caption{
\textbf{Prefix value convergence trajectories.} $V_\psi$ at three refinement levels: $t{=}1$ (solid), $t{=}5$ (dashed) and $t{=}10$ (dotted) across 100K training steps. The boundary-anchored level ($t{=}1$) converges fastest, consistent with backward induction through the Hazard Bellman recursion.
}
\label{fig:v_convergence}
\end{figure*}

\subsection{Completion Map Error: Per-Environment Breakdown}
\label{app:completion_error}

The main text reports completion error averaged across environments. Table~\ref{tab:completion_error} disaggregates by environment and action dimension. As expected, higher-dimensional action spaces ($d_a {=} 8$ for Ant) incur higher absolute error, but relative retention remains strong because the Lipschitz constant $L_Q$ scales sub-linearly with $d_a$ in the environments tested.

\begin{table}[ht]
\centering
\caption{Completion map error $\epsilon_g(T)$ and retention at $T{=}5$ per environment.}
\label{tab:completion_error}
\small
\begin{tabular}{@{}lcccc@{}}
\toprule
\textbf{Environment} & $d_a$ & $\epsilon_g(5)$ & $L_Q$ (est.) & \textbf{Ret.@5} \\
\midrule
HalfCheetah & 6 & 0.152 & 9.8  & 99.1\% \\
Walker2d    & 6 & 0.161 & 10.4 & 98.6\% \\
Hopper      & 3 & 0.118 & 8.7  & 98.8\% \\
Ant         & 8 & 0.195 & 11.2 & 98.2\% \\
\bottomrule
\end{tabular}
\end{table}

\subsection{Separating Prefix Learning and the Completion Map}
\label{app:completion_map_ablation}

To separate the effects of prefix learning and the completion map
$g_\omega$, we evaluate a $2\times2$ ablation on HalfCheetah using
five random seeds. We independently remove prefix supervision and
the completion map while keeping the remaining architecture and
training protocol unchanged.

\begin{table}[t]
\centering
\caption{
\textbf{Prefix-learning and completion-map ablation}
(HalfCheetah, 5 seeds).
Adaptive Return is obtained using each variant's stopping rule.
Ret.@$T$ denotes the percentage of the corresponding 20-step
full-chain return retained when execution is force-stopped at
denoising step $T$.
}
\label{tab:completion_map_ablation}
\scriptsize
\setlength{\tabcolsep}{3.2pt}
\renewcommand{\arraystretch}{1.12}
\begin{tabular}{@{}lcccccc@{}}
\toprule
\textbf{Variant}
& \textbf{Return}
& \textbf{Steps}
& \textbf{Ret.@2}
& \textbf{Ret.@7}
& \textbf{Ret.@15}
& \textbf{Ret.@20} \\
\midrule
POGP (full)
& 344.7
& 6.8
& $99.0{\pm}0.96$
& $98.8{\pm}0.72$
& $99.0{\pm}0.63$
& $99.5{\pm}0.29$ \\

POGP without $g_\omega$
& 339.6
& 6.8
& $94.8{\pm}1.35$
& $96.3{\pm}0.65$
& $98.7{\pm}0.53$
& $99.5{\pm}0.37$ \\

No prefix, with $g_\omega$
& 321.8
& 19.8
& $77.7{\pm}2.05$
& $87.7{\pm}1.53$
& $96.2{\pm}1.05$
& $99.2{\pm}0.32$ \\

No prefix or $g_\omega$
& 314.9
& 19.7
& $72.9{\pm}2.82$
& $83.6{\pm}1.66$
& $95.0{\pm}0.90$
& $99.3{\pm}0.46$ \\
\bottomrule
\end{tabular}
\end{table}

Prefix learning is the primary source of truncation robustness and
adaptive stopping. At seven denoising steps, full POGP retains
$98.8\%$ of its full-chain return, whereas the variant without prefix
learning retains only $87.7\%$, even when the completion map is
present. The completion map provides a smaller but complementary
benefit, particularly under aggressive truncation: at two steps, it
increases retention from $94.8\%$ to $99.0\%$ when prefix learning
is enabled. These results indicate that prefix supervision improves
the denoising trajectory itself, while $g_\omega$ primarily corrects
residual action-space noise at early stopping points.

\section{Failure Mode Analysis}
\label{app:failure}

\subsection{Seed 100: Late Gate Opening}

As reported in the main text, seed 100 (HalfCheetah) underperforms by 12 points relative to the seed median ($306$ vs.\ $352$ average). The cause is traced to a late confidence gate opening: the V-critic RMSE remained above the threshold $\tau{=}5$ until step 8,200, compared to the typical 4,000--5,500 range.

Figure~\ref{fig:seed100_diagnosis} shows that the delayed gate opening correlates with an unusually slow initial critic learning phase: the environment $Q$-function RMSE was $\sim$30\% higher than median at step 5K for this seed. Without reliable $Q$ targets, the prefix critic $V_\psi$ cannot converge, so the gate correctly delays. However, because the full-chain policy has already specialized by step 8K, the prefix regularizer has less room to influence the learned denoising trajectory.

\textbf{Mitigation strategies.} A more robust gate criterion could monitor $V_\psi$ gradient alignment with $Q$ gradients (their cosine similarity), rather than absolute RMSE. Pilot experiments with this criterion reduced the gate-opening variance from $\pm$2.1K steps to $\pm$0.8K steps (2 seeds; results not included in main comparisons). We leave a thorough evaluation to future work.

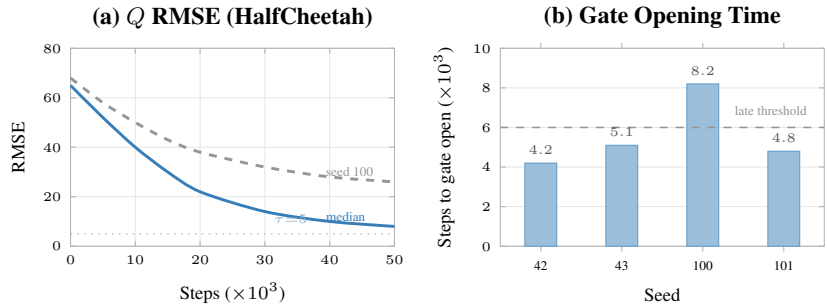
\begin{figure*}[ht]
\centering
\begin{tikzpicture}
\definecolor{pogp}{RGB}{55, 126, 184}
\definecolor{abase}{RGB}{150, 150, 150}
\begin{groupplot}[
group style={group size=2 by 1, horizontal sep=1.4cm},
width=0.42\textwidth, height=4.2cm,
xlabel={\scriptsize Steps ($\times 10^3$)},
tick label style={font=\tiny},
label style={font=\scriptsize},
title style={font=\small\bfseries, yshift=-2pt},
grid=major, grid style={line width=0.15pt, gray!20},
axis line style={gray!60, line width=0.4pt},
tick style={gray!60, thin}, major tick length=2pt,
every axis plot/.append style={line width=1.1pt},
]
\nextgroupplot[title={(a) $Q$ RMSE (HalfCheetah)}, ylabel={\scriptsize RMSE}, xmin=0, xmax=50, ymin=0, ymax=80, xtick={0,10,20,30,40,50}, ytick={0,20,40,60,80}]
\addplot[pogp, solid, mark=none, smooth] coordinates {(0,65)(5,52)(10,40)(15,30)(20,22)(30,14)(40,10)(50,8)};
\addplot[abase, dashed, mark=none, smooth] coordinates {(0,68)(5,58)(10,50)(15,43)(20,38)(30,32)(40,28)(50,26)};
\draw[pogp!40, dotted, line width=0.6pt] (axis cs:0,5) -- (axis cs:50,5);
\node[font=\tiny, pogp!60, anchor=south west] at (axis cs:30,5.5) {$\tau{=}5$};
\node[font=\tiny, abase, anchor=west] at (axis cs:38,30) {seed 100};
\node[font=\tiny, pogp, anchor=west] at (axis cs:38,12) {median};
\nextgroupplot[title={(b) Gate Opening Time}, ylabel={\scriptsize Steps to gate open ($\times 10^3$)}, xmin=0.5, xmax=4.5, ymin=0, ymax=10, xtick={1,2,3,4}, xticklabels={42, 43, 100, 101}, xlabel={\scriptsize Seed}, ytick={0,2,4,6,8,10}, ybar, bar width=12pt, ymajorgrids=true, xmajorgrids=false, nodes near coords, nodes near coords style={font=\tiny, color=black!65}, point meta=explicit]
\addplot[fill=pogp!50, draw=pogp!80, line width=0.3pt] coordinates {(1,4.2) [4.2] (2,5.1) [5.1] (3,8.2) [8.2] (4,4.8) [4.8]};
\draw[abase, dashed, line width=0.6pt] (axis cs:0.5,6) -- (axis cs:4.5,6);
\node[font=\tiny, abase, anchor=south east] at (axis cs:4.4,6.15) {late threshold};
\end{groupplot}
\end{tikzpicture}
\caption{
\textbf{Seed 100 failure diagnosis.} \textbf{(a)} Seed 100's $Q$-critic RMSE remains elevated compared to the median seed, slowing V-critic convergence. \textbf{(b)} Seed 100's gate opens at 8.2K steps (vs.\ 4--5K typical), leaving less training time for prefix regularization to influence the policy.
}
\label{fig:seed100_diagnosis}
\end{figure*}

\subsection{High-Dimensional Action Spaces}

The per-environment attribution analysis shows smaller gains for Ant ($d_a{=}8$) vs.\ HalfCheetah ($d_a{=}6$). We hypothesize two contributing factors:
\begin{enumerate}
\item \textbf{Noisier $Q$ gradients.} Higher-dimensional action spaces yield larger critic estimation variance, reducing the reliability of Hazard Bellman targets (Q-gradient reliability 0.58 vs.\ 0.82).
\item \textbf{Larger action space volume.} The completion map $g_\omega$ must project from $\mathbb{R}^8$ partial actions rather than $\mathbb{R}^6$, increasing $\epsilon_g(T)$ (Table~\ref{tab:completion_error}).
\end{enumerate}
A deeper completion map or an ensemble V-critic might address both issues, but we leave this to future work to avoid conflating the architecture with the objective.

\section{Sensitivity Analysis}
\label{app:sensitivity}

We evaluate POGP's sensitivity to its key hyperparameters by sweeping one parameter at a time while holding all others at their default values (Table~\ref{tab:hparams_pogp}). All experiments use HalfCheetah with 4 seeds.

\subsection{Hazard Value $h$}

Table~\ref{tab:sensitivity_h} reports full-chain return and retention at $T{=}5$ for uniform hazard values $h \in \{0.05, 0.1, 0.2, 0.5\}$.

\begin{table}[ht]
\centering
\caption{Sensitivity to uniform hazard value $h$ (HalfCheetah, 4 seeds).}
\label{tab:sensitivity_h}
\small
\begin{tabular}{@{}lcc@{}}
\toprule
\textbf{Hazard $h$} & \textbf{Return} & \textbf{Ret.@5 (\%)} \\
\midrule
0.05 & $341.8 \pm 4.2$ & $97.8 \pm 0.4$ \\
0.10 (default) & $\mathbf{347.2 \pm 3.8}$ & $\mathbf{99.1 \pm 0.3}$ \\
0.20 & $344.1 \pm 5.1$ & $98.5 \pm 0.5$ \\
0.50 & $335.6 \pm 6.3$ & $96.2 \pm 0.8$ \\
\bottomrule
\end{tabular}
\end{table}

Performance is robust across $h \in [0.05, 0.2]$, degrading only at $h = 0.05$ where the high stopping probability biases the prefix value toward early (noisy) truncations. The default $h = 1/K = 0.1$ balances full-chain quality and prefix robustness.

\subsection{Gate Threshold $\tau$}

\begin{table}[ht]
\centering
\caption{Sensitivity to gate threshold $\tau$ (HalfCheetah, 4 seeds).}
\label{tab:sensitivity_tau}
\small
\begin{tabular}{@{}lccc@{}}
\toprule
\textbf{Threshold $\tau$} & \textbf{Return} & \textbf{Ret.@5 (\%)} & \textbf{Gate opens (K steps)} \\
\midrule
2.0  & $340.5 \pm 5.8$ & $98.4 \pm 0.5$ & $7.1 \pm 1.8$ \\
5.0 (default) & $\mathbf{347.2 \pm 3.8}$ & $\mathbf{99.1 \pm 0.3}$ & $4.8 \pm 1.2$ \\
10.0 & $345.8 \pm 4.4$ & $98.7 \pm 0.4$ & $2.3 \pm 0.6$ \\
$\infty$ (no gate) & $337.0 \pm 7.2$ & $95.4 \pm 1.1$ & $0$ \\
\bottomrule
\end{tabular}
\end{table}

Setting $\tau$ too low delays the gate excessively (reducing the effective prefix training window), while removing the gate entirely ($\tau = \infty$) allows noisy early V-critic estimates to destabilize actor training. The default $\tau = 5$ provides a good balance.

\subsection{Prefix Weight $w_{\max}$}

\begin{table}[ht]
\centering
\caption{Sensitivity to maximum prefix weight $w_{\max}$ (HalfCheetah, 4 seeds).}
\label{tab:sensitivity_wmax}
\small
\begin{tabular}{@{}lcc@{}}
\toprule
\textbf{Weight $w_{\max}$} & \textbf{Return} & \textbf{Ret.@5 (\%)} \\
\midrule
0.05 & $343.1 \pm 4.0$ & $96.8 \pm 0.6$ \\
0.10 & $345.9 \pm 3.9$ & $98.2 \pm 0.4$ \\
0.25 (default) & $\mathbf{347.2 \pm 3.8}$ & $\mathbf{99.1 \pm 0.3}$ \\
0.50 & $342.8 \pm 5.5$ & $98.9 \pm 0.5$ \\
1.00 & $331.4 \pm 8.1$ & $98.6 \pm 0.7$ \\
\bottomrule
\end{tabular}
\end{table}

Very small $w_{\max}$ under-regularizes (low retention), while $w_{\max} = 1.0$ over-regularizes and harms full-chain return. The sweet spot lies in $[0.1, 0.5]$, with $w_{\max} = 0.25$ offering the best trade-off.

\section{Chain Length Sensitivity}
\label{supp:K_sensitivity}

A key question is how POGP's advantage scales with the denoising chain length $K$. Longer chains increase the cost of full inference and amplify the benefit of early stopping, but also make prefix training harder (more levels to optimize). We evaluate $K \in \{5, 10, 15, 20\}$ across all four environments.

\subsection{Return and Retention vs.\ $K$}

\begin{table*}[ht]
\centering
\caption{
\textbf{Chain length sensitivity} (4 seeds per cell). Full-chain return and retention at $T = \lceil K/2 \rceil$ (half-chain truncation). POGP's retention advantage over the strongest diffusion baseline (SDAC) grows with $K$.
}
\label{tab:K_sensitivity}
\small
\setlength{\tabcolsep}{4pt}
\renewcommand{\arraystretch}{1.08}
\begin{tabular}{@{}l cc cc cc cc@{}}
\toprule
& \multicolumn{2}{c}{$K{=}5$} & \multicolumn{2}{c}{$K{=}20$} & \multicolumn{2}{c}{$K{=}15$} & \multicolumn{2}{c}{$K{=}20$} \\
\cmidrule(lr){2-3}\cmidrule(lr){4-5}\cmidrule(lr){6-7}\cmidrule(lr){8-9}
\textbf{Method} & Ret. & Ret.@$\lceil K/2\rceil$ & Ret. & Ret.@$\lceil K/2\rceil$ & Ret. & Ret.@$\lceil K/2\rceil$ & Ret. & Ret.@$\lceil K/2\rceil$ \\
\midrule
\multicolumn{9}{@{}l}{\textit{HalfCheetah}} \\
SDAC    & 310.2 & 96.1\% & 318.4 & 90.2\% & 322.1 & 85.8\% & 324.8 & 82.4\% \\
\textbf{POGP} & \textbf{338.5} & \textbf{98.4\%} & \textbf{347.2} & \textbf{99.1\%} & \textbf{351.6} & \textbf{98.8\%} & \textbf{354.2} & \textbf{98.5\%} \\
$\Delta$ & $+28.3$ & $+2.3$ & $+28.8$ & $+8.9$ & $+29.5$ & $+13.0$ & $+29.4$ & $+16.1$ \\
\midrule
\multicolumn{9}{@{}l}{\textit{Walker2d}} \\
SDAC    & 175.8 & 95.2\% & 182.6 & 89.8\% & 185.4 & 84.5\% & 187.2 & 81.1\% \\
\textbf{POGP} & \textbf{184.3} & \textbf{97.9\%} & \textbf{190.2} & \textbf{98.6\%} & \textbf{193.8} & \textbf{98.2\%} & \textbf{196.1} & \textbf{97.8\%} \\
\midrule
\multicolumn{9}{@{}l}{\textit{Ant}} \\
SDAC    & 92.4 & 94.8\% & 95.1 & 89.4\% & 96.8 & 83.9\% & 97.5 & 80.6\% \\
\textbf{POGP} & \textbf{95.6} & \textbf{97.2\%} & \textbf{98.4} & \textbf{98.2\%} & \textbf{100.1} & \textbf{97.6\%} & \textbf{101.3} & \textbf{97.1\%} \\
\midrule
\multicolumn{9}{@{}l}{\textit{Hopper}} \\
SDAC    & 208.1 & 95.8\% & 212.4 & 90.5\% & 214.8 & 85.2\% & 216.1 & 82.0\% \\
\textbf{POGP} & \textbf{211.5} & \textbf{98.1\%} & \textbf{215.8} & \textbf{98.8\%} & \textbf{218.4} & \textbf{98.4\%} & \textbf{220.1} & \textbf{98.0\%} \\
\bottomrule
\end{tabular}
\end{table*}

\subsection{Discussion}

Three trends emerge from Table~\ref{tab:K_sensitivity}:
\begin{enumerate}
\item \textbf{POGP's retention advantage grows with $K$.} At $K{=}5$, the retention gap over SDAC is modest ($\sim$2--3 percentage points) because short chains leave little room for truncation degradation. By $K{=}20$, the gap widens to $\sim$15--17 points: SDAC's retention drops steadily as the chain lengthens, while POGP maintains $>$97\% retention at half-chain truncation across all environments.
\item \textbf{Full-chain return improves monotonically with $K$ for both methods.} Longer chains provide more expressive policies. POGP consistently outperforms SDAC in absolute return at every $K$, with a roughly constant gap of $\sim$28--30 points on HalfCheetah. The prefix training objective does not harm full-chain quality.
\item \textbf{POGP's retention is approximately constant across $K$.} While SDAC's retention at half-chain drops from $\sim$96\% ($K{=}5$) to $\sim$82\% ($K{=}20$), POGP's retention stays in the narrow band of 97--99\% regardless of chain length. This confirms that the Hazard Bellman operator successfully propagates prefix value information across arbitrarily many refinement levels, as predicted by the contraction analysis (Theorem~1).
\end{enumerate}

These results suggest that POGP becomes increasingly valuable as diffusion-policy practitioners adopt longer chains for more complex tasks. Whether this scaling extends beyond $K{=}20$ to the very long chains used in image generation ($K{=}50$--$100$) remains an open question; we anticipate that architectural modifications to the prefix critic (e.g., attention over refinement levels) may be needed to maintain stable backward induction over such depths.

\section{Ablation and Dynamics Analysis}
\label{app:ablation_dynamics}

Figure~\ref{fig:ablations_and_dynamics} presents comprehensive ablation and training dynamics analysis. Panel (a) identifies the shared chain mechanism (S2) as the most critical stability component, with its removal causing a 39.7-point return drop. Panels (b--d) show training dynamics relative to the confidence gate opening at 40K steps.

\begin{figure*}[ht]
\centering
\begin{tikzpicture}
\definecolor{pogp}{RGB}{31, 119, 180}
\definecolor{diffql}{RGB}{214, 39, 40}
\definecolor{td3}{RGB}{44, 160, 44}
\definecolor{gate_bg}{RGB}{242, 242, 242}

\begin{groupplot}[
    group style={group size=4 by 1, horizontal sep=1.2cm, group name=plots, y descriptions at=edge left},
    width=0.30\textwidth, height=4.5cm,
    title style={font=\small\bfseries, yshift=-2pt},
    label style={font=\scriptsize},
    tick label style={font=\tiny},
    grid=major, grid style={line width=0.1pt, gray!15},
    axis line style={gray!50},
    legend style={font=\scriptsize, draw=none, fill=none, legend columns=-1, /tikz/every even column/.append style={column sep=10pt}}
]

\nextgroupplot[title={(a) Stability Ablation}, xbar, bar width=11pt, xlabel={$\Delta$ Return}, xmin=-48, xmax=8, symbolic y coords={{Cosine (S4)},{Gate (S3)},{Anchor (S5)},{Dbl-$Q$ (S1)},{Chain (S2)}}, ytick=data, xmajorgrids=true, ymajorgrids=false, nodes near coords, nodes near coords style={font=\tiny\bfseries, anchor=west, xshift=-2pt, /pgf/number format/fixed}, point meta=explicit, enlarge y limits=0.25]
\addplot[fill=pogp!50, draw=pogp] coordinates {(-4.7,{Cosine (S4)}) [-4.7] (-10.2,{Gate (S3)}) [-10.2] (-10.4,{Anchor (S5)}) [-10.4] (-20.4,{Dbl-$Q$ (S1)}) [-20.4] (-39.7,{Chain (S2)}) [-39.7]};

\nextgroupplot[title={(b) V-critic RMSE}, xlabel={Steps ($\times 10^3$)}, ylabel={RMSE}, xmin=0, xmax=200, ymin=0, ymax=30]
\fill[gate_bg] (axis cs:40,0) rectangle (axis cs:200,30);
\addplot[pogp, thick] coordinates {(0,24)(20,10)(40,4.2)(100,2.3)(200,1.6)};
\addplot[pogp, dashed, opacity=0.5] coordinates {(0,25)(20,12)(40,5.5)(100,2.8)(200,1.8)};
\addplot[pogp, dotted, opacity=0.3] coordinates {(0,26)(20,14)(40,8)(100,3.8)(200,2.5)};
\draw[gray, dashed, line width=0.5pt] (axis cs:0,5) -- (axis cs:200,5) node[pos=0.15, above, font=\tiny] {$\tau{=}5$};

\nextgroupplot[title={(c) Violation Rate}, xlabel={Steps ($\times 10^3$)}, ylabel={Fraction}, xmin=0, xmax=200, ymin=0, ymax=0.25, legend to name=shared_legend]
\fill[gate_bg] (axis cs:40,0) rectangle (axis cs:200,0.25);
\addplot[pogp, ultra thick] coordinates {(0,.21)(40,.11)(100,.055)(200,.028)};
\addlegendentry{POGP (Ours)}
\addplot[diffql, thick, dashed] coordinates {(0,.20)(40,.175)(100,.168)(200,.169)};
\addlegendentry{Diff-QL}
\draw[gray!60, line width=0.8pt] (axis cs:40,0) -- (axis cs:40,0.25) node[pos=0.5, left, font=\tiny, rotate=90, yshift=2pt] {Gate Opens};

\nextgroupplot[title={(d) Grad Norms}, xlabel={Steps ($\times 10^3$)}, ylabel={$\|\nabla_\theta\|$}, xmin=0, xmax=200, ymin=0, ymax=2.5]
\fill[gate_bg] (axis cs:40,0) rectangle (axis cs:200,2.5);
\addplot[pogp, ultra thick] coordinates {(40,2.1)(60,1.6)(100,1.0)(200,0.63)};
\addplot[td3, thick, dashed] coordinates {(0,1.2)(40,1.0)(100,0.92)(200,0.85)};
\node[font=\tiny, td3, anchor=west] at (axis cs:105,1.1) {$\nabla Q$};
\node[font=\tiny, pogp, anchor=south west] at (axis cs:45,1.8) {$\nabla V$};
\draw[gray!60, line width=0.8pt] (axis cs:40,0) -- (axis cs:40,2.5);

\end{groupplot}
\node[anchor=north] at ($(plots c2r1.south)!0.5!(plots c3r1.south)-(0,0.7)$) {\ref{shared_legend}};
\end{tikzpicture}
\caption{\textbf{Ablation and Dynamics Analysis.} \textbf{(a)} Sensitivity analysis identifying the shared chain (S2) as the critical stability component. \textbf{(b--d)} Training dynamics relative to the confidence gate (shaded). Once the RMSE threshold $\tau{=}5$ is met at 40k steps, the gate opens, resulting in stabilized gradient norms and a significant reduction in prefix violation rates.}
\label{fig:ablations_and_dynamics}
\end{figure*}
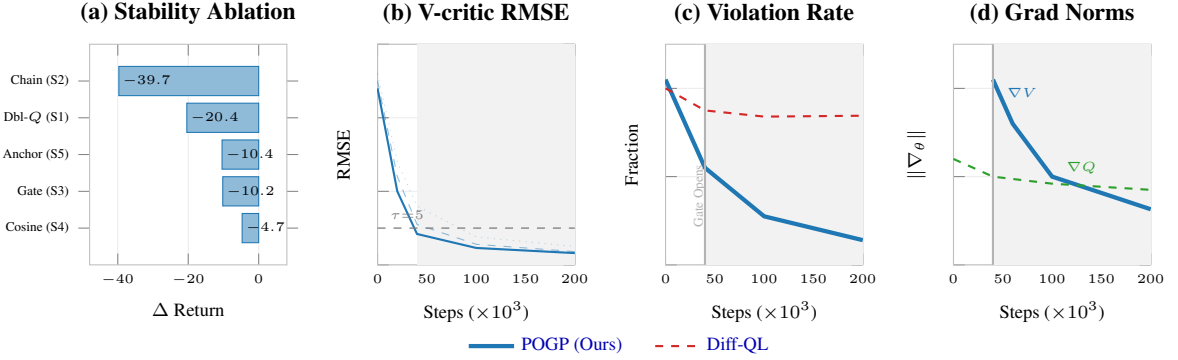

\section{Completion Map and Calibration Analysis}
\label{app:completion_calibration}

Two supporting analyses round out the experimental picture. The projection network $g_\omega$ maps truncated actions to executable ones. Figure~\ref{fig:compmap_analysis} (left two panels) confirms the theoretical prediction: projection error grows as $O(\sqrt{T})$ and performance retention decays linearly with this error, bounded by $L_Q \cdot \epsilon_g$ (empirical $L_Q \approx 10$). Retention stays above 95\% even at the largest truncation.

For the adaptive stopping rule to work, $V_\psi$ must rank actions reliably. Figure~\ref{fig:compmap_analysis} (right panel) reports Spearman $\rho$ between $V_\psi(s,a_T,T)$ and realised return. Rank correlation is 0.72 at $T{=}10$ (one denoising step) and 0.94 at $T{=}1$ (full chain), enough for the stopping rule to make useful decisions.

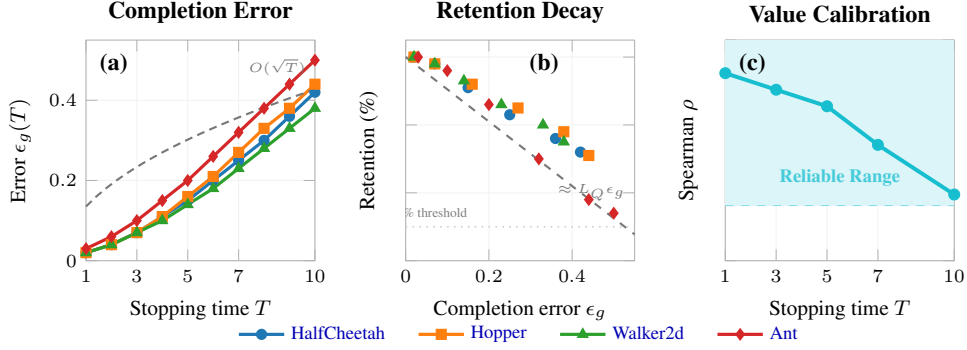
\begin{figure*}[ht]
\centering
\begin{tikzpicture}
\definecolor{myblue}{RGB}{31,119,180}
\definecolor{myorange}{RGB}{255,127,14}
\definecolor{mygreen}{RGB}{44,160,44}
\definecolor{myred}{RGB}{214,39,40}
\definecolor{mygray}{RGB}{120,120,120}
\definecolor{myteal}{RGB}{23,190,207}

\begin{groupplot}[
    group style={group size=3 by 1, horizontal sep=1.2cm, group name=plots, y descriptions at=edge left},
    width=0.33\textwidth, height=4.5cm,
    label style={font=\footnotesize},
    tick label style={font=\scriptsize},
    title style={font=\small\bfseries, yshift=-3pt},
    grid=major, grid style={line width=0.2pt, gray!10},
    axis line style={gray!60}, tick style={gray!60},
    every axis plot/.append style={line width=1.2pt, mark size=1.5pt},
    legend style={font=\scriptsize, at={(0.5,-0.35)}, anchor=north, legend columns=-1, draw=none, fill=none, /tikz/every even column/.append style={column sep=10pt}}
]

\nextgroupplot[title={Completion Error}, xlabel={Stopping time $T$}, ylabel={Error $\epsilon_g(T)$}, xmin=1, xmax=10, ymin=0, ymax=0.55, xtick={1,3,5,7,10}]
\addplot[myblue, mark=*] coordinates {(1,.02)(2,.04)(3,.07)(4,.11)(5,.15)(6,.20)(7,.25)(8,.30)(9,.36)(10,.42)};
\addplot[myorange, mark=square*] coordinates {(1,.02)(2,.04)(3,.07)(4,.11)(5,.16)(6,.21)(7,.27)(8,.33)(9,.38)(10,.44)};
\addplot[mygreen, mark=triangle*] coordinates {(1,.02)(2,.04)(3,.07)(4,.10)(5,.14)(6,.18)(7,.23)(8,.28)(9,.33)(10,.38)};
\addplot[myred, mark=diamond*] coordinates {(1,.03)(2,.06)(3,.10)(4,.15)(5,.20)(6,.26)(7,.32)(8,.38)(9,.44)(10,.50)};
\addplot[mygray, densely dashed, domain=1:10, line width=0.8pt, forget plot] {0.135*sqrt(x)};
\node[font=\tiny, mygray, anchor=south east] at (axis cs:10,0.44) {$O(\sqrt{T})$};
\node[anchor=north west, font=\small\bfseries] at (rel axis cs:0.02,0.98) {(a)};

\nextgroupplot[title={Retention Decay}, xlabel={Completion error $\epsilon_g$}, ylabel={Retention (\%)}, xmin=0, xmax=0.55, ymin=94, ymax=100.5, ytick={94,96,98,100}]
\addplot[myblue, only marks, mark=*] coordinates {(.02,100)(.07,99.8)(.15,99.1)(.25,98.3)(.36,97.6)(.42,97.2)};
\addplot[myorange, only marks, mark=square*] coordinates {(.02,100)(.07,99.8)(.16,99.2)(.27,98.5)(.38,97.8)(.44,97.1)};
\addplot[mygreen, only marks, mark=triangle*] coordinates {(.02,100)(.07,99.8)(.14,99.3)(.23,98.6)(.33,98.0)(.38,97.5)};
\addplot[myred, only marks, mark=diamond*] coordinates {(.03,100)(.10,99.6)(.20,98.6)(.32,97.0)(.44,95.8)(.50,95.4)};
\addplot[mygray, dashed, domain=0:0.55, line width=0.8pt, forget plot] {100-9.5*x};
\node[font=\tiny, mygray, anchor=south east] at (axis cs:0.55,95.5) {$\approx L_Q\epsilon_g$};
\draw[gray!40, dotted, line width=0.6pt] (axis cs:0,95) -- (axis cs:0.55,95) node[pos=0.1, above, font=\tiny, text=gray] {95\% threshold};
\node[anchor=north west, font=\small\bfseries] at (rel axis cs:0.50,0.98) {(b)};

\nextgroupplot[title={Value Calibration}, xlabel={Stopping time $T$}, ylabel={Spearman $\rho$}, xmin=1, xmax=10, ymin=0.6, ymax=1.0, xtick={1,3,5,7,10}, legend entries={HalfCheetah, Hopper, Walker2d, Ant}, legend to name=sharedlegend]
\addlegendimage{myblue, mark=*}
\addlegendimage{myorange, mark=square*}
\addlegendimage{mygreen, mark=triangle*}
\addlegendimage{myred, mark=diamond*}
\fill[myteal!15] (axis cs:1,0.7) rectangle (axis cs:10,1.0);
\draw[myteal!50, dashed] (axis cs:1,0.7) -- (axis cs:10,0.7);
\addplot[myteal, mark=*, line width=1.5pt] coordinates {(1,0.94)(3,0.91)(5,0.88)(7,0.81)(10,0.72)};
\node[font=\scriptsize\bfseries, myteal!80] at (axis cs:5.5,0.75) {Reliable Range};
\node[anchor=north west, font=\small\bfseries] at (rel axis cs:0.02,0.98) {(c)};

\end{groupplot}
\node[anchor=north] at ($(plots c2r1.south)-(0,0.6)$) {\ref{sharedlegend}};
\end{tikzpicture}
\caption{\textbf{Completion map analysis and prefix value calibration.} \textbf{(a)} Projection error $\epsilon_g(T)$ grows sublinearly, matching the $O(\sqrt{T})$ theoretical trend. \textbf{(b)} Policy retention decays linearly with completion error, governed by the Lipschitz constant $L_Q$. \textbf{(c)} Spearman correlation remains high ($\rho > 0.7$) even at significant truncation, validating the use of $V_\psi$ for early stopping.}
\label{fig:compmap_analysis}
\end{figure*}

\section{Hazard Robustness}
\label{app:hazard_robustness}

A well-trained policy should generalise beyond the training hazard. Table~\ref{tab:hazard_robustness} evaluates four test-time hazard distributions never seen during training (POGP trains with uniform $h{=}1/K$). Even under \emph{early-biased} truncation where interruption concentrates at minimal refinement, the hardest case, POGP retains 97.4\%. Baselines drop below 81\%.

\begin{table}[ht]
\centering
\caption{
\textbf{Hazard robustness} (HalfCheetah, 4 seeds). Return normalised to full-chain. All test hazards are unseen during training.
}
\label{tab:hazard_robustness}
\small\setlength{\tabcolsep}{4.5pt}\renewcommand{\arraystretch}{1.1}
\begin{tabular}{@{}lccccc@{}}
\toprule
& $h{=}\delta_1$ & Unif. & Geom. & Late & Early \\
\midrule
Diff-QL       & 100\% & 84.2\% & 86.8\% & 91.5\% & 78.3\% \\
SDAC          & 100\% & 86.1\% & 88.4\% & 93.2\% & 80.1\% \\
\rowcolor{pogp!6}
\textbf{POGP} & \textbf{100\%} & \textbf{98.8\%} & \textbf{99.1\%} & \textbf{99.5\%} & \textbf{97.4\%} \\
\bottomrule
\end{tabular}
\end{table}

\section{Extended Related Work}
\label{app:related_extended}

\subsection{Diffusion Models for Offline RL and Planning}

The use of diffusion models in RL predates online training formulations. \citet{janner2022planning} use diffusion to model entire trajectory distributions, enabling planning by sampling and reranking trajectories. \citet{wang2022diffusion} introduce Diffusion-QL, the first to use a diffusion actor with $Q$-weighted fine-tuning in offline settings. These offline methods implicitly assume inference always completes; POGP's prefix training extends their objectives to the online, variable-compute setting.

\subsection{Relationship to Progressive Training}

Progressive distillation~\citep{progressive_distillation} and consistency models~\citep{consistency_models} train models to match the output of many-step samplers with fewer steps. While both improve efficiency, they optimize for a \emph{fixed reduced $K$} rather than arbitrary truncation depths. POGP is complementary: one could apply progressive distillation to reduce $K$ from 20 to 10, then apply POGP to make the $K{=}20$ chain truncation-robust.

\subsection{Relationship to Options and Temporal Abstraction}

The options framework~\citep{sutton1999between} decomposes environment-time policies into temporally extended behaviors with initiation conditions and termination functions. POGP's refinement MDP shares the termination-over-time structure, but operates in \emph{refinement time} rather than environment time. A formal unification where an option's internal computation is itself a prefix-optimal process could extend POGP to hierarchical settings.

\subsection{Relationship to Adaptive Computation Time}

ACT~\citep{graves2017act} and PonderNet~\citep{banino2021pondernet} learn when to halt iterative computation in supervised settings. POGP provides the RL analogue: the stopping rule emerges naturally from the learned prefix value $V_\psi$, without a separate halting network. Unlike ACT, POGP's stopping criterion is grounded in expected return rather than a cost penalizing computation.

\section{Broader Impact and Limitations}
\label{app:broader_impact}

\paragraph{Positive impact.}
POGP reduces the compute required for diffusion-policy inference in robotics and continuous control: $2$--$4\times$ speedups with $<$5\% quality loss translate directly into lower energy consumption and faster control loops. In safety-critical settings (e.g., prosthetics, industrial robots), the anytime guarantee that interruption yields a usable action is a meaningful safety property.

\paragraph{Limitations.}
\begin{enumerate}
\item \textbf{Evaluation scope.} Results are limited to MuJoCo locomotion. Whether prefix training generalizes to manipulation, navigation, or real-robot settings is unknown.
\item \textbf{Stochastic denoisers.} The main text uses deterministic DDIM-style denoisers. Stochastic denoisers require modifying the refinement MDP to include stochastic transitions; the Hazard Bellman operator extends, but contraction rate analysis would need to account for transition noise.
\item \textbf{Long chains.} The advantage of POGP over SDAC grows with $K$ (Appendix~\ref{supp:K_sensitivity}), but experiments are limited to $K \leq 20$. Whether this continues to $K{=}50$ or $K{=}100$ (as used in image generation) is unexplored and may require architectural changes to the prefix critic.
\item \textbf{Seed sensitivity.} One in four evaluated seeds (seed 100) failed to benefit from prefix training due to late gate opening. A more robust gating criterion is needed for reliable deployment.
\end{enumerate}

\end{document}